\documentclass[ejs]{imsart}

\usepackage{amsthm,amsmath,amssymb,amsfonts}
\RequirePackage[round,authoryear]{natbib}
\usepackage{wrapfig}
\usepackage{graphicx}

\newtheorem{theorem}{Theorem}
\newtheorem{definition}{Definition}
\newtheorem{cor}{Corollary}
\newtheorem{prop}{Proposition}
\newtheorem{example}{Example}

\doi{10.1214/154957804100000000}
\pubyear{0000}
\volume{0}
\firstpage{1}
\lastpage{1}
\startlocaldefs
\endlocaldefs

\begin{document}

\begin{frontmatter}

% "Title of the Paper"
\title{\Large Uphill Roads to Variational Tightness: Monotonicity and Monte Carlo Objectives}
\runtitle{Monte Carlo Objectives}

% indicate corresponding author with \corref{}
% \author{\fnms{John} \snm{Smith}\thanksref{t1}\corref{}\ead[label=e1]{smith@foo.com}\ead[label=e2,url]{www.foo.com}}
% \thankstext{t1}{Thanks to somebody} 
% \address{line 1\\ line 2\\ \printead{e1}\\ \printead{e2}}

\author{\fnms{Pierre-Alexandre} \snm{Mattei}\ead[label=e1]{pierre-alexandre.mattei@inria.fr}}
\address{Université Côte d'Azur \\Inria, Maasai project-team \\Laboratoire J.A. Dieudonné, CNRS \\ \printead{e1}}
%\and
\author{\fnms{Jes} \snm{Frellsen}\ead[label=e2]{jefr@dtu.dk}}
\address{Department of Applied Mathematics and Computer Science\\  Technical University of Denmark\\ \printead{e2}}

\runauthor{Mattei \& Frellsen}

\begin{abstract}
	We revisit the theory of importance weighted variational inference (IWVI), a promising strategy for learning latent variable models. IWVI uses new variational bounds, known as Monte Carlo objectives (MCOs), obtained by replacing intractable integrals by Monte Carlo estimates---usually simply obtained via importance sampling.
	\citet{burda2016} showed that increasing the number of importance samples provably tightens the gap between the bound and the likelihood. Inspired by this simple monotonicity theorem, we present a series of nonasymptotic results that link properties of Monte Carlo estimates to tightness of MCOs. We challenge the rationale that smaller Monte Carlo variance leads to better bounds. We confirm theoretically the empirical findings of several recent papers by showing that, in a precise sense, negative correlation reduces the variational gap. We also generalise the original monotonicity theorem by considering non-uniform weights. We discuss several practical consequences of our theoretical results. Our work borrows many ideas and results from the theory of stochastic orders.
\end{abstract}

%\begin{keyword}[class=MSC]
%\kwd[Primary ]{}
%\kwd{}
%\kwd[; secondary ]{}
%\end{keyword}

%\begin{keyword}
%\kwd{}
%\kwd{}
%\end{keyword}

% history:
% \received{\smonth{1} \syear{0000}}

%\tableofcontents

\end{frontmatter}

\section{Introduction}

Often, objective  functions that arise in machine learning applications involve seemingly intractable high-dimensional integrals. Examples include likelihood-based inference of latent variable models (because the likelihood can be written as an integral over the latent space) or models with unnormalised densities (a.k.a.\@ energy-based models, because they involve intractable normalising constants), hard attention problems (because they involve marginalising over all possible ``glances'' of the observations, see e.g.\@ \citealp{ba2015learning}), or information-theoretic  representation learning (see e.g.\@ \citealp{alemi2017}).
\emph{Variational inference} constitutes a toolbox of techniques that tackle this issue by replacing the objective function to maximise by a lower bound of it (that is supposed to be easier to compute and/or optimise).

A recent and promising approach to variational inference was proposed by \citet{burda2016}, notably building on prior work by \citet{bornschein2014reweighted}. The idea is simply to replace the intractable integrals by Monte Carlo estimates of it, and optimise the expected value of this approximation with respect to both model parameters and the randomness induced by the Monte Carlo approximation. Following \citet{mnih2016}, these new bounds are called \emph{Monte Carlo objectives (MCOs)}, and are typically obtained using importance sampling with a parametrised posterior that can be optimised. This new flavour of variational inference is usually called \emph{importance weighted variational inference (IWVI)}. 

While they were originally developed to learn unsupervised deep latent variable models similar to variational autoencoders (VAEs, \citealp{kingma2014,rezende2014}), MCOs have been successfully applied to a diverse family of problems, including inference for Gaussian processes  \citep{salimbeni2019deep}, sequential models \citep{maddison2017,naesseth2018variational,anh2018} or exponential random graphs \citep{tan2020bayesian}, missing data imputation \citep{mattei2019,ipsen2021not}, causal inference \citep{josse2020}, neural spike inference \citep{speiser2017}, dequantisation \citep{hoogeboom2020learning}, verification of deep discriminative models \citep{che2020}, and general Bayesian inference \citep{domke2018,domke2019}.

These empirical successes have been calling for theoretical developments. For example, a natural question is then: \emph{how do properties of the Monte Carlo estimate translate into properties the variational bound?} This question, which is the main topic of this paper, has until now mostly been tackled from an \emph{asymptotic} point of view. More specifically, most results are concerned with the behaviour of MCOs when the number of Monte Carlo samples go to infinity (or when the variance goes to zero). This contrasts with the fact that, in practice, the number of samples rarely exceeds a few dozens (for computational reasons), and the variance is very large if not infinite. Motivated by this gap between theory and practice, our perspective here is \emph{non-asymptotic}. One exception to the asymptotic focus is the beautifully simple result proven by \citet{burda2016}: when the weights are exchangeable, \emph{increasing the number of samples always improves the tightness of the bound}. This monotonicity theorem, which we will refer to as \emph{sample monotonicity}, is the main inspiration of our paper.

\subsection{Contributions and organisation of the paper}

We begin by a review of the theory and applications of MCOs (Section 2), that we use to justify a quite general mathematical framework. Then, we explore the interplay between MC properties and variational tightness by gradually increasing the strength of the assumtions:
\begin{itemize}
	\item We start in Section 3 by considering general estimates (not necessarily based on importance sampling). This leads us to challenge the popular heuristic that reducing MC variance improves the bound. We propose the stronger notion of \emph{convex order} in order to control the tightness  of the bound.
	\item We then consider general multiple importance sampling, with potentially different proposals (Section 4), and leverage the theory of stochastic orders to show that negative dependence provably leads to better bounds (Theorem 1), confirming theoretically several recent works.
	\item Adding the assumption that the weights are exchangeable, we present a generalisation of the monotonicity theorem of \citet{burda2016}, that we trace back to \citet{marshall1965inequality}.
\end{itemize}
Along the way, we discuss some practical consequences of our theoretical results.

\section{Variational inference using Monte Carlo objectives}

In this section, we present the general context of variational inference using Monte Carlo objectives. We briefly review prior work and present a general mathematical framework to analyse these objectives.

We consider some data $\mathbf{x} \in \mathcal{X}$ governed by a latent variable $\mathbf{z} \in \mathcal{Z}$ through a model with density
\begin{equation}
	\label{eq:model}
	p(\mathbf{x},\mathbf{z}) = p (\mathbf{z})p(\mathbf{x}|\mathbf{z}),
\end{equation}
with respect to a dominating measure on $\mathcal{X} \times \mathcal{Z}$. We use densities because they are more conventionally used in the latent variable model literature, although a more general measure-theoretic framework could also be contemplated, in the fashion of \citet[Appendix B]{domke2019}. The model \eqref{eq:model} may, or may not be Bayesian, depending on whether or not unknown parameters are included in the latent variable $\mathbf{z}$.

\subsection{Inference via Monte Carlo objectives}

Typically, the latent variable models we focus on depend on many parameters that we would like to learn via (potentially approximate) maximum likelihood. Since $\mathbf{z}$ is hidden and only $\mathbf{x}$ is observed, the log-likelihood (or log-marginal likelihood if the model is Bayesian) is equal to
\begin{equation}
	\ell = \log p(\mathbf{x}) = \log \int_\mathcal{Z} p(\mathbf{x|\mathbf{z}}) p (\mathbf{z}) d\mathbf{z}.
\end{equation}

A fruitful idea to approach $\ell$ is to replace the typically intractable integral $p(\mathbf{x})$ inside the logarithm by a Monte Carlo estimate of it. Of particular interest are unbiased estimates, since they lead to lower bounds of the likelihood $\ell$. Indeed, if $R$ is a random variable such that $R>0$ and $\mathbb{E}[R] = p(\mathbf{x})$, then the quantity $\mathcal{L} = \mathbb{E}[\log R]$ is a lower bound of the likelihood $\ell$, by virtue of Jensen's \citeyearpar{jensen1905,jensen1906} inequality and the concavity of the logarithm.
Moreover, the fact that, in $\mathcal{L}$, the expectation is now located \emph{outside} of the logarithm means that $\mathcal{L}$ is more suited for stochastic optimisation techniques (which require unbiased estimates of the gradient of the objective function, see e.g.~\citealp{bottou2018}). The lower bound $\mathcal{L}$ is called a \emph{Monte Carlo objective (MCO)}, and is usually maximised in lieu of the likelihood. 

In this paper, we will study in particular importance sampling estimates of the form
\begin{equation}
	\label{eq:is}
	R_K = \frac{1}{K}  \sum_{k=1}^K \frac{p(\mathbf{x}|\mathbf{z}_k)p(\mathbf{z}_k)}{q(\mathbf{z}_k | \mathbf{x})},
\end{equation}
where $\mathbf{z}_1,\ldots,\mathbf{z}_K$ follow a \emph{proposal distribution} $q(\mathbf{z}_1,\ldots,\mathbf{z}_K | \mathbf{x})$ that usually is a function of the data $\mathbf{x}$ (e.g.\@ via a neural network, as in VAEs). The corresponding MCO is then  $\mathcal{L}_K = \mathbb{E} [\log R_K]$, which may be optimised using stochastic optimisation.

\subsection{A brief history of MCOs and IWVI}

Using importance sampling to approximate a (marginal) likelihood is quite an old idea (see e.g.\@ \citealp{geweke1989bayesian}, and references herein). The idea to \emph{use this approximation as an objective function for inference} is not new either: for example \emph{Monte Carlo maximum likelihood} (MCML, see e.g.\@ \citealp{geyer1994convergence}, and references herein) is a popular inference technique that aims at maximising the approximation $\log R$ of the likelihood. Until now, it seems that the connections between MCML and MCOs have not been discussed in the literature. For this reason, let us spend a few lines on this.
Essentially, MCML differs from the MCO approach in three ways:
\begin{itemize}
	\item the objective function of MCML is the random quantity $\log R$, while a MCO is a deterministic function $\mathbb{E}[\log R]$,
	\item  a MCO is generally jointly optimised over both the model parameters, i.e.\@ the parameters of the distribution $p(\mathbf{x})$, and the parameters of the proposal distribution $q$, while MCML generally separates the two steps,
	\item MCOs have deep connections with variational inference, and can be interpreted as divergences between the posterior distribution of the latent variables and an approximation of it \citep{domke2018,domke2019}.
\end{itemize}
The \emph{reweighted wake-sleep (RWS)} algorithm of  \citet{bornschein2014reweighted} is one step closer to IWVI. The idea of RWS is to repeat the following steps:
\begin{itemize}
	\item wake-phase: the bound $\mathcal{L}$ is optimised with respect to the model parameters,
	\item sleep-phase: the proposal $q(\mathbf{z}|\mathbf{x})$ is optimised by minimising its Kullback-Leibler divergence to the true posterior $p(\mathbf{z}|\mathbf{x})$.
\end{itemize}
Both steps generally involve approximate optimisation by performing a few stochastic gradient steps. The main difference between RWS and IWVI is that RWS involves two objective functions to be optimised alternatively, and IWVI maximises a single objective: the MCO. Interesting discussions on the links between RWS and IWVI include \citet{dieng2019reweighted}, \citet{finke2019importance}, \citet{le2020revisiting}, and \citet{kim2020}.

An important point that we  will not explore in this paper is the optimisation, of MCOs. Naive stochastic gradient descent may encounter severe problems, in particular when using many samples---this issue, and some remedies, are explored for example by \citet{rainforth18}, \citet{tucker2018}, and \citet{lievin2020optimal}.
Regarding more applied advances, multiple references of successful applications of MCOs are listed in the beginning of the introduction of this paper (in a wide variety of domains, including e.g.\@ causal inference, missing data imputation, Gaussian process inference, or neural imaging).

The main question that motivates this paper is:
\emph{What are the properties of the function $(K,q) \mapsto \mathcal{L}_K(q) $?} In particular, we would like to know how changing $q$ and $K$ will affect the \emph{likelihood gap} $\mathcal{G}_K (q) = \ell - \mathcal{L}_K(q) $.

Quite a large body of work has been devoted to studying the asymptotics of $\mathcal{G}_K (q)$, as we will detail in Section \ref{sec:variance_heuristic}.
Our perspective here is quite different. In the spirit of the original monotonicity result of \citet{burda2016}, we wish to obtain \emph{non-asymptotic guarantees} about the behaviour of $\mathcal{L}_K(q)$ and $\mathcal{G}_K (q)$ when $K$ and $q$ vary.

\subsection{General setting and notations}

Motivated by the questions above, we focus on the following formal context, which is slightly more general than the one described above.

We consider a potentially infinite sequence of positive random variables $\mathbf{w} = (w_k)_{k \in\mathcal{K} } $ with common mean $\mu > 0$. This sequence, called the \emph{sequence of importance weights}, is indexed by $\mathcal{K} = \{1,\ldots,K_{\textup{max}}\}$, where $K_{\textup{max}} \in \mathbb{N}^* \cup \{\infty\}$. The joint distribution of $\mathbf{w}$ is denoted by $Q$. Note that we do not make any assumption on $Q$ yet (it may have diverse marginals, not be factored, not be absolutely continuous). For all $K \in\mathcal{K}$, the simple Monte Carlo estimate of $\mu > 0$ is $R_K= S_K/K$, where $S_K = w_1 +\ldots+w_K$.
The \emph{sequence of Monte Carlo objectives} $\boldsymbol{\mathcal{L}}(Q) = (\mathcal{L}_K (Q))_{K \in\mathcal{K}} $, is defined by 
\begin{align}
	\mathcal{L}_K (Q) &=  \mathbb{E}_Q\left[ \log \left( \frac{1}{K} \sum_{k=1}^{K}w_k  \right) \right] = \mathbb{E}_Q\left[R_K\right] =  \mathbb{E}_Q\left[ \log S_K \right] - \log K.
\end{align}
It is possible to be slightly more general by replacing the uniform coefficients $1/K,\ldots,1/K$ by a vector $\boldsymbol{\alpha}$ in the $K$-simplex $\Delta_K$. This leads to
\begin{equation}
	\mathcal{L}_{\boldsymbol{\alpha}}(Q) = \mathbb{E}_Q\left[ \log \left(\boldsymbol{\alpha}^T \mathbf{w}  \right) \right] = \mathbb{E}_Q\left[ \log \left(\sum_{k=1}^{K}\alpha_k w_k  \right) \right].
\end{equation}
In particular $\mathcal{L}_{(1/K,\ldots,1/K)}(Q) = \mathcal{L}_{K}(Q)$. Jensen's  \citeyearpar{jensen1905,jensen1906} inequality ensures that, since the logarithm is concave, $\mathcal{L}_{\boldsymbol{\alpha}}(Q) \leq \log \mu$. Note however that it is possible to have $ \mathcal{L}_{\boldsymbol{\alpha}}(Q) = - \infty$ (we will show an example of this in the next section). We may also consider random coefficients $\boldsymbol{\alpha} \sim \nu$, where $\nu$ is a distribution over the simplex $\Delta_K$. In this context, we have $\mathcal{L}_{\boldsymbol{\alpha}}(Q) = \mathbb{E}_{Q}\left[ \log \left(\boldsymbol{\alpha}^T \mathbf{w}  \right) |\boldsymbol{\alpha}   \right]$ and this leads to the bound $\mathbb{E}_\nu[\mathcal{L}_{\boldsymbol{\alpha}}(Q)] \leq \log \mu$.

In the context of latent variable models, 
$\mu = p(\mathbf{x})$; $\mathbf{w}$ is the sequence of importance weights; for all $ K \in\mathcal{K}$, the distribution of $(w_k)_{k \leq K}$ is the push-forward of the proposal $q(\mathbf{z}_1,\ldots,\mathbf{z}_K)$ by the mapping 
\begin{equation*}
	(\mathbf{z}_1,\ldots,\mathbf{z}_K) \mapsto \left(\frac{p(\mathbf{x}|\mathbf{z}_1)p(\mathbf{z}_1)}{q(\mathbf{z}_1)},\ldots, \frac{p(\mathbf{x}|\mathbf{z}_K)p(\mathbf{z}_K)}{q(\mathbf{z}_K)}\right);
\end{equation*}
and $R_K$ is the unbiased estimate of the likelihood $p(\mathbf{x})$ defined by importance sampling, as in Equation \eqref{eq:is}. The non-uniform version $\mathcal{L}_{\boldsymbol{\alpha}}$ corresponds to using \emph{multiple importance sampling} (see e.g.\@ \citealp{elvira2019generalized}, for a general review). To stress the fact that $R_K$ depends on $\mathbf{x}$ and approximates $p(\mathbf{x})$, we will also note it $\hat{p}(\mathbf{x})$ instead of $R_K$ in sometimes (e.g.\@ in Example \ref{ex:div}).

We believe that this simple but general framework covers most ways of defining importance-sampling based MCOs, from the original ones of \citet{burda2016}, corresponding to i.i.d.\@ weights with uniform coefficients, to the more elaborated ones of  \citet{huang2019}, where the weights are correlated and not identically distributed, and notably statistically dependent on their coefficients $\boldsymbol{\alpha}$.

\subsection{Warm-up: sample monotonicity}

As an illustration of the kinds of monotonicity results we wish to prove, let us start by re-stating the sample monotonicity result of \citet{burda2016}, which is the main inspiration of this paper.

\begin{theorem}[\textbf{sample monotonicity}, \citealp{burda2016}]
	\label{th:sample_mono}
	If Q is exchangeable, then $\boldsymbol{\mathcal{L}}(Q)$ is nondecreasing, i.e.\@ for all $ K  \in \{1,\ldots, K_\textup{max} - 1\}$,
	\begin{equation}
		\mathcal{L}_K (Q) \leq \mathcal{L}_{K+1} (Q).
	\end{equation}
\end{theorem}
We remind that exchangeability means that permuting the indices of the weights does not change their distribution. More specifically, for any permutation $\sigma \in S(\mathcal{K})$, $(w_{\sigma(k)})_{ k \in\mathcal{K}}$ and $(w_{k})_{k \in\mathcal{K}}$ are identically distributed.

The version of Theorem \ref{th:sample_mono} that we presented here is slightly more general than the one of \citet{burda2016}, who assumed that the weights are i.i.d.\@ (which is a strictly stronger condition than exchangeability). Nonetheless, their proof also works under exchangeability. Another interesting preliminary remark about the proof of \citet{burda2016} is that their reasoning remains valid if the logarithm is replaced by any other concave function. Most of the results of our paper will share this general property.

While exchangeability is weaker than the i.i.d.\@ assumption, it is stronger than just assuming that the weights are identically distributed (i.d.). A first natural question pertaining generalisations of sample monotonicity is therefore: \emph{is it sufficient to have i.d.\@ weights?} The answer is no, as shown by the following simple counter-example.

\begin{example}[\textbf{i.d.\@ is not enough}]
	\label{eq:exch_necessary}
	Let $x$, $y$ be i.i.d.\@ positive random variables. Using  the identically distributed (but non-exchangeable) weights $w_1 =x$, $w_2 =y$,  $w_3 =x$ leads to $\mathcal{L}_2(Q)  \geq  \mathcal{L}_3(Q) $.
\end{example}
Example \ref{eq:exch_necessary} will turn out to a be consequence of a generalisation of sample monotonicity with non-uniform coefficients, see Equation \eqref{eq:proof_exch_necessary}.

\section{Variance reduction as a heuristic towards tighter bounds}

Variance reduction is often considered as the simplest way of improving Monte Carlo estimates. It sounds then natural to assume that variance reduction will lead to tighter bounds. We revisit this rationale here, and challenge it.

\label{sec:variance_heuristic}

\subsection{The variance heuristic}
%\citet{maddison2017,klys2018joint,domke2018,huang2019note}

At its simplest level, what we call the \emph{variance heuristic} may be informally formulated like this: \emph{in a MCO,  if $\textup{Var}(R)$ gets smaller, then $R$ is a more accurate estimate of $\mathbb{E}[R] = \mu$, and the variational bound $\mathbb{E}[\log R]$ gets tighter}. It is possible to be more formal by Taylor-expanding the logarithm of $R$ around $\mu$:

\begin{equation}
	\log(R) = \log(\mu + (R- \mu) ) = \log(\mu)  + \frac{R- \mu}{\mu} -  \frac{(R- \mu)^2}{2 \mu ^2} + \textup{Rem}(R).
\end{equation}
The Taylor remainder $\textup{Rem}(R)$ may be for example written using its integral form
\begin{equation}
	\textup{Rem}(R,\mu) =  - \int_{\mu}^R  \frac{(R-t)^2}{2  \mu ^2 t^2}dt.
\end{equation}
Then, assuming that $\textup{Var}(R)$ is finite, computing the expectation leads to

\begin{equation}
	\label{eq:taylor}
	\mathbb{E}[\log(R)] =  \log(\mu)  +  \frac{\textup{Var}(R)}{2 \mu^2} + \mathbb{E}[\textup{Rem}(R,\mu)].
\end{equation}
The variance heuristic can then be seen as a consequence of the assumption that, in Equation \eqref{eq:taylor}, the variance term dominates the remainder. In other words, it can be seen as \emph{second order heuristic}.
There are good reasons to believe that this assumption is reasonable when $R$ is very concentrated around $\mu$ (e.g.\@ when $\textup{Var}(R)$ is small). This is the rationale behind the results of
\citet[Proposition 1]{maddison2017}, \citet[Proposition 1]{nowozin2018debiasing}, \citet{klys2018joint}, \citet[Theorem 3]{domke2019}, \citet{huang2019note}, and \citet[Chapter 3]{dhekane2021improving}. Similar ideas (in a setting more general than the one of MCOs) are also present in \citet[Theorem 3]{rainforth2018nesting}. \citet[Section 4]{huang2019} also suggested to look at $\textup{Var}[\log R]$ as an asymptotic indication of  tightness of the bound.

Let us see what might sometimes break in this line of reasoning. First, we have no guarantee that the variance is actually finite. It is even quite common to encounter infinite variance importance sampling estimates, and we will give empirical evidence that the ones commonly used in VAEs have indeed infinite variance. Even assuming that the variance is finite, there are many situations where we could expect the Taylor remainder to be non-negligible. Indeed, the radius of convergence of the logarithm as a power series is quite small (the radius of  $x \mapsto \log(x) $ is $\mu$ around $\mu$). This means that even a high order heuristic will not be accurate if $R$ gets far away from its mean $\mu$.

\subsection{Simple successes, simple failures}

Sample monotonicity can be seen a first example of success of the variance heuristic: adding more importance weights will both reduce Monte Carlo variance and tighten the bound.

\begin{example}[\textbf{sample monotonicity and variance reduction}]
	\label{prop:sm_variance}
	Let $R_K$ be the importance sampling estimate. Let us assume that the weights are exchangeable and have finite variance. Sample monotonicity ensures that the MCO will increase. However, adding samples will also have a variance reduction effect: for all $ K  \in \{1,..., K_\textup{max} - 1\}$,
	\begin{equation}
		\textup{Var}(R_{K+1}) \leq \textup{Var}(R_{K}).
	\end{equation}
	In the case of i.i.d. weights, this simply follows from $\textup{Var}(R_{K})=\textup{Var}(w_1)/K^2$. In the exchangeable case, this can be shown directly or seen as a consequence of a generalised version of sample monotonicity (Theorem \ref{th:marshall}). We will see that, in fact, these two simultaneous monotonicity properties (of the bound and of the variance) are different sides of the same coin.
\end{example}
Let us now look at the general case where $R$ can be any unbiased Monte Carlo estimate (not necessarily obtained via importance sampling). Of course, this is an overly general setting, and some assumptions must be made in order to be able to prove something. For example, we may wonder what happens when $R$ beyond to simple families of distributions. Sometimes, things will go as foretold by the heuristic, as seen below.

\begin{example}[\textbf{a few successes of the variance heuristic}]
	\label{prop:successes}
	Let $R$ and $R'$ be either two gamma, two inverse gamma, or two log-normal distributions with finite and equal means and finite variances. Then
	\begin{equation}
		\label{eq:successes}
		\textup{Var}[R] < \textup{Var}[R'] \iff \textup{Var}[\log R] < \textup{Var}[\log R'] \iff  \mathbb{E}[\log R] > \mathbb{E}[\log R'].
	\end{equation}
\end{example}

The proof is available in Appendix A. The fact that these are exponential families suggests that a more general result may be hidden behind Proposition \ref{prop:successes}. While interesting in its own right, such a result would not be particularly relevant in the context of MCOs. Indeed, in general, with MCOs, $R$ follows a complex distribution very unlikely to belong to an exponential family.

What does it take to violate the heuristic using these kinds of simple distributions?  While comparing two inverse gammas or two log-normals always respects it, simply blending these two family is enough to get severe violations.

\begin{example}[\textbf{severe failure of the variance heuristic}]
	\label{ex:failures}
	Let $R$ be an inverse-gamma variable with finite mean. It is possible to find a log-normal random variable $R'$ such that
	\begin{itemize}
		\item $\mathbb{E}[R] = \mathbb{E}[R']$, $\textup{Var}[R] = \infty$, $\textup{Var}[R'] < \infty$,
		\item $\mathbb{E}[\log R] > \mathbb{E}[\log R']$.
	\end{itemize}
\end{example}

Again, the proof is available in Appendix A. In particular, we show that the gap $\mathbb{E}[\log R] - \mathbb{E}[\log R']$ can be made arbitrarily large by choosing the log-normal parameters (im)properly. This means that, when comparing MCOs, it is possible to be in a situation where \emph{infinitely worse variance leads to an arbitrarily better bound}.
It is also possible to be in a situation that is somehow the opposite of Example \ref{ex:failures}: the variance is finite, but the bound is not.

\begin{example}[\textbf{finite variance, infinitely loose bound}]
	It is possible to find random variables such that $\mathbb{E}[R]= \mu $, $\mathbb{E}[\log R] = - \infty$, and $\textup{Var}(R)<+\infty$. This is for example the case of the finite moment log-stable distributions of \citet{carr2003finite}. This family is constituted of some exponentiated Lévy stable distributions (for more details, see also \citealp{robinson2015}).
\end{example}

\begin{wrapfigure}{R}{0.35\textwidth}
	\centering
	\vspace{-7mm}
	\label{fig:var}
	\includegraphics[width = 0.3\columnwidth, trim = {1.5cm 28cm 61cm 0.5cm}, clip]{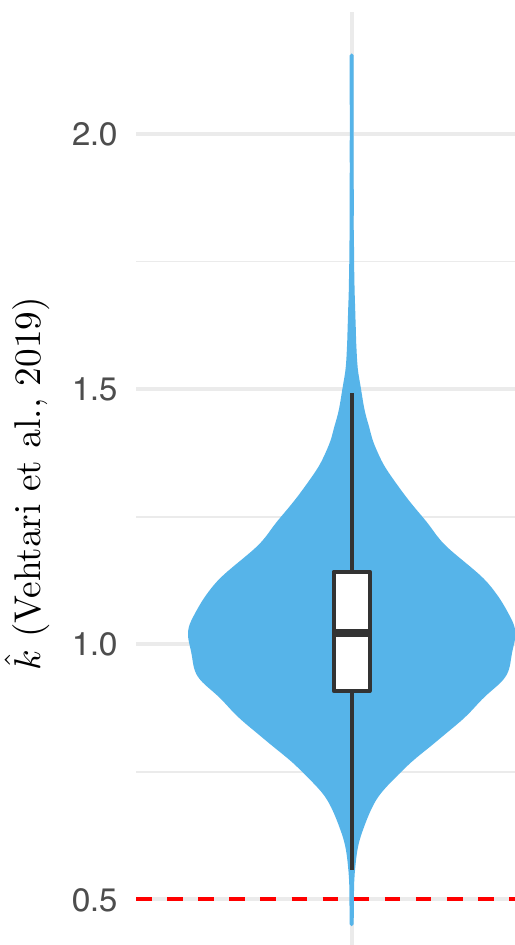}
	\caption{Importance sampling diagnostics for a VAE on MNIST. To each training digit corresponds a value of $\hat{k}$. Values of $\hat{k}$ above the dashed line correspond to digits whose weights have potentially infinite variance.}
	\vspace{-15mm}
\end{wrapfigure}

While it is not very surprising to find counter examples of these sorts, it is interesting to see that such severe failures may be observed using quite simple distributions.  This phenomenon is reminiscent of the line of thought of \citet{chatterjee2018}, who argued that the variance is not a very good metric for devising good importance sampling estimates.

\subsection{Is the variance finite in practice?}

It is often the case that importance weights have infinite variance. We provide empirical evidence that this is the case in the simple case of a VAE trained on MNIST (Figure 1). After training, we compute $10{,}000$ weights for each digits that we use to compute the $\hat{k}$ diagnostic of \citet{vehtari2019pareto}.  Most digits have a $\hat{k}>0.5$, and are therefore suspect of having infinite variance. This illustrates again the shortcomings of the variance. More details on this experiment are provided in Appendix B.

\subsection{Beyond variance reduction: the convex order}

While a powerful heuristic, variance reduction is not a strong enough dispersion measure to guarantee the tightening of a bound. Such a measure is provided by a branch of the theory of \emph{stochastic orders} (extensively reviewed in the monograph of \citealp{shaked2007}). The essential idea is  to define binary relations $\preceq$ between distributions (or equivalently random variables) such that $R \preceq R'$ means that, in some sense, $R$ is more concentrated that $R'$. A popular dispersion order is the \emph{convex order} (reviewed for example by \citealp{shaked2007}, Section 3.A).

\begin{definition}
	Let $R \sim \pi_1 $ and $R' \sim \pi_2 $ be two univariate random variables. We say that $R$ is smaller than $R'$ in the convex order if
	\begin{equation}
		\label{eq:cvx}
		\mathbb{E}_{\pi_1}[\phi (R )] \leq \mathbb{E}_{\pi_2}[\phi (R' )],
	\end{equation}
	for all convex functions $\phi$ such that the involved expectations exist. We denote  $R \preceq_\textup{CX} R'$ or $\pi_1 \preceq_\textup{CX}\pi_2$.
\end{definition}

Contrarily to variance reduction, which does not provide general tightening guarantees, convex ordering implies both variance reduction and bound tightening.

\begin{prop}
	Let $R, R' $ be two univariate random variables. We have
	\begin{equation}
		R \preceq_\textup{CX} R' \implies \left\{\begin{matrix}
			\mathbb{E}[\log (R )] \geq \mathbb{E}[\log (R' )] \\ \textup{Var}(R) \leq   \textup{Var}(R') \\ \mathbb{E}[R] = \mathbb{E}[R']
		\end{matrix}\right..
	\end{equation}
\end{prop}
\begin{proof}
	These are direct consequences of the  concavity of the logarithm, and the convexity of $x \mapsto x^2$, $x \mapsto x$, and $x \mapsto -x$.
\end{proof}
This simple result means that \emph{the more concentrated the Monte Carlo estimate (in the sense of the convex order), the tighter the bound} and the smaller the variance. We will see in the next sections that two successful ways of provably getting tighter bounds can be explained from the perspective of the convex order: using more importance weights, and increasing their negative dependence. In these cases, \emph{variance reduction will be merely seen as a side effect of convex domination}.

\paragraph{Is the convex order too strong?} The fact that the inequality \eqref{eq:cvx} needs to hold for every convex function $\phi$ appears like quite a strong condition. Indeed, it would be sufficient to have this for any class of functions containing $-\log$ to be able to control the tightness of the bound. For example, one might only consider decreasing convex functions. This would lead to considering \emph{monotonic convex orders}, other dispersion orders  which are less popular than the convex order, but has some useful properties (see e.g.~\citealp{shaked2007}, Chapter 4). In the specific case of MCOs, convex order and increasing convex order are the same. Indeed, one can show \citep[Theorem 4.A.35]{shaked2007} that, if $\mathbb{E}[R] = \mathbb{E}[R']$,  then 
\begin{multline}
	R \preceq_\textup{CX} R' \iff \mathbb{E}[\phi (R )] \leq \mathbb{E}[\phi (R' )] \textup{ for all decreasing convex $\phi$}\\ \textup{such that both expectations exist}.
\end{multline} 
This means that the convex order is weaker than it looks. Of course, it is still stronger than just looking at the variance.  Beyond a more mathematically convenient framework, what would we gain from this larger generality? The next example provides a simple illustration in the context of latent variable models.

\begin{example}[Divergence control]
	\label{ex:div}
	Let us go back to the context of latent variable models. Here, the unbiased estimate $R$, seen as a function of $\mathbf{x}$, may be viewed as an approximation of the true density of the model $p(\mathbf{x})$. To highlight this, we will denote $R = \hat{p}(\mathbf{x})$. Now how far is $\hat{p}$ from $p$? A natural way of quantifying this is to use probability divergences, for exemple $f$-divergences. For some fixed smooth convex function $f$ such that $f(1)=0$, the $f$-divergence between the density $p_1$ of a finite measure and a probability density $p_2$ is defined as
	\begin{equation}
		D_f(p_1 || p_2) = \int_\mathcal{X} f\left(\frac{p_1(\mathbf{x})}{p_2(\mathbf{x})} \right) p_2(\mathbf{x})d\mathbf{x} + f'(1) \left(\int_\mathcal{X} p_1(\mathbf{x}) d \mathbf{x} -1 \right).
	\end{equation}
	Particular cases of $f$-divergences include e.g.~the popular Kullback Leibler (KL) divergence and its ``reverse" version or the squared Hellinger distance, depending on the choice of $f$. In this version of the definition of $f$-divergences (seen, e.g., in \citealp{stummer2010divergences}), the first density $p_1$ does not have to sum to one, which is fortunate because $\hat{p}$ typically will not. Indeed, the only thing we can guarantee is that $\hat{p}$ sums on average to one:
	\begin{equation}
		\mathbb{E}\left[ \int_\mathcal{X} \hat{p}(\mathbf{x}) d\mathbf{x} \right] =\int_\mathcal{X} \mathbb{E}\left[ \hat{p}(\mathbf{x})  \right] d\mathbf{x}  = \int_\mathcal{X} p(\mathbf{x}) d\mathbf{x} = 1.
	\end{equation}	
	This implies that $\hat{p}$ almost surely corresponds to a finite measure. Therefore, the quantity $D_f(\hat{p} || p)$, which is random because $\hat{p}$ is random, is almost surely well-defined. Its average value will then be
	\begin{align}
		\mathbb{E}[D_f(\hat{p} || p)] &= \mathbb{E}\left[\int_\mathcal{X} f\left(\frac{\hat{p}(\mathbf{x})}{p(\mathbf{x})} \right) p(\mathbf{x})d\mathbf{x}\right] + f'(1) \left(\mathbb{E}\left[\int_\mathcal{X} \hat{p}(\mathbf{x}) d \mathbf{x}\right] -1 \right) \\
		&= \int_\mathcal{X} \mathbb{E}\left[f\left(\frac{\hat{p}(\mathbf{x})}{p(\mathbf{x})} \right)\right] p(\mathbf{x})d\mathbf{x}.
	\end{align} 
	Finally, using the convexity of $f$, we can conclude that
	\begin{equation}
		\label{eq:f-div}
		\left\{ \forall \mathbf{x} \in \mathcal{X}, \;\hat{p}_1(\mathbf{x}) \preceq_\textup{CX} \hat{p}_2(\mathbf{x}) \right\}\implies  \mathbb{E}[ D_f(\hat{p}_1 || p)]  \leq \mathbb{E}[D_f(\hat{p}_2 || p)],
	\end{equation}
	which means that the approximation of the distribution becomes better and better when the MC approximation gets more concentrated.
\end{example}	
This example shows that convex domination can be useful in the LVM context beyond controlling bound tightness. We will use \eqref{eq:f-div} later in the paper to show some monotonicty properties of $f$-divergences in the specific case of importance sampling estimates.

\section{Negative dependence and tighter bounds}

A popular branch of variance reduction techniques is based on \emph{negative dependence}. In its simplest form, this idea is based on the fact that 
\begin{equation}
	\label{eq:negative_cov}
	\textup{Var}(\boldsymbol{\alpha}^T \mathbf{w}) = \sum_{k=1}^K \alpha_k^2 \textup{Var}(w_k) + 2\sum_{1 \leq k <k' \leq K} \alpha_k \alpha_{k'}\textup{Cov}(w_k,w_{k'}),
\end{equation}
which means that negative covariances in the right hand side of Equation \eqref{eq:negative_cov} will lead to a smaller variance of $\boldsymbol{\alpha}^T \mathbf{w}$. \emph{Anthitetic sampling} is for example a famous variance reduction technique based on this idea (see e.g.~\citealp{mcbook}, Section 8.2). A more refined approach is to leverage \emph{determinantal point processes} \citep{bardenet2020monte}.

Variants of this rationale were used successfully in the MCO context by  \citet{klys2018joint}, \citet{huang2019}, \citet{ren2019}, \citet{wu2019}, and \citet{domke2019}. Their motivations were essentially based on variants of the variance heuristic: since negative dependence can reduce the variance, it might also improve the bound. Our goal here is to prove that negative dependence can indeed tighten the bound, giving hereby a non-asymptotic theoretical justification for the works aforementioned.

\subsection{Comparing dependence with the supermodular order}

Let $\mathbf{w} \sim Q_1$ and $\mathbf{v} \sim Q_2$  be two $K$-dimensional random variables with identical marginals, i.e.~$w_k \overset{d}{=} v_k$ for all $k \in \{1,...,K\}$ (some would say that $Q_1$ and $Q_2$ belong to the same Fréchet class, see e.g.~). What mathematical sense could we give to the sentence ``the coordinates of $\mathbf{w}$ are more negatively dependent than those of $\mathbf{v}$"? Again, stochastic orders provide good tools for assessing this. Indeed, the idea of \emph{dependence orders} is  to define binary relations $\preceq$ between distributions such that $Q_1 \preceq Q_2$ means that, in some sense, the coordinates of $\mathbf{w} \sim Q_1$ are more negatively dependent that those of $\mathbf{v}\sim Q_2$. We will  review in this section  a few of these dependence-based stochastic orders (a more detailed overview may be found in \citealp{shaked2007}, Chapter 9, or  \citealp{ruschendorf2013mathematical}, Chapter 6). We focus on the supermodular order, which, as we will see in the next subsection, is the most closely related to Monte Carlo objectives. We first need to define supermodular functions.

\begin{definition}
	A function $\phi: \mathbb{R}^K\rightarrow \mathbb{R}$ is \textbf{supermodular} if, for all $\mathbf{x}, \mathbf{y} \in \mathbb{R}^K$,
	\begin{equation}
		\label{eq:sm}
		\phi(\min(\mathbf{x},\mathbf{y})) + \phi(\max(\mathbf{x},\mathbf{y})) \geq \phi(\mathbf{x})+\phi(\mathbf{y}).
	\end{equation}
\end{definition}
In Equation \ref{eq:sm}, the min and max functions are applied elementwise. We can now define the supermodular order.
%For more details on supermodular functions, see for example \citet{bach}.
\begin{definition}
	Let $Q_1$ and $Q_2$ be two probability distributions over $\mathbb{R}^K$. We say that \textbf{$Q_1$ is smaller than $Q_2$ in the supermodular order} when $$\mathbb{E}_{Q_1}[\phi(\mathbf{w})] \leq \mathbb{E}_{Q_2}[\phi(\mathbf{v})] $$ for all supermodular functions $\phi$ such that the involved expectations  exist. We denote  $Q_1 \preceq_\textup{SM} Q_2$.
\end{definition}
As a first remark, note that $Q_1 \preceq_\textup{SM} Q_2$ implies that $Q_1$ and $Q_2$ have identical marginals.

The supermodular order is one of the most popular stochastic orders when it comes to quantify dependence (see e.g.~\citealp{muller2000,shaked2007}, Chapter 9), notably in the economics and insurance literature (see e.g.~\citealp{muller1997stop,meyer2012increasing,meyer2013}).  \citet[Section 2.2.3]{joe1997} proposed a set of nine axioms that would characterise good dependence orders. A few years later, \citet{muller2000} proved that the supermodular order satisfied all of these desirable properties.

Here is a simple example of supermodular ordering: for two distributions $Q_1, Q_2$ with identical marginals, if the coordinates of $\mathbf{w} \sim Q_1$ are negatively associated, and those of  $\mathbf{v} \sim Q_2$ are independent, then $Q_1 \preceq_\textup{SM} Q_2$ \citep{christofides2004}.

%\paragraph{Why is the supermodular order a good way to compare dependence?}

%The reasons for this popularity may be classified in two branches: the own merits of the supermodular order (which are reviewed in this subsection), and its numerous relationships with other dependence orders (outlined in Subsection ?).

%It is worth acknowledging that the link between supermodular functions and dependence might not seem obvious at first sight. Some valuable insight can be gained by looking at the bivariate case ($K=2$). Indeed...

%\paragraph{Supermodular order and variance.} As outlined in Section ?, negative dependence is usually seen as a variance reduction technique. An interesting feature of the supermodular order is that it quantifies how negative dependence leads to variance reduction. Indeed, since the function $ \mathbf{w} \mapsto \sum_{k=1}^K w_k^2  $ is supermodular (because its Hessian is diagonal),
%\begin{equation}
%Q_1\preceq_\textup{SM} Q_2 \implies \textup{Var} \left( \sum_{k=1}^K w_k \right) \leq  \textup{Var} \left( \sum_{k=1}^K v_k \right).
%\end{equation}
%This means that the variance of the weights gets smaller when they get more negatively dependent in the supermodular sense (a similar remark was made by \citealp{christofides2004}, Example 2). Does this variance reduction effect leads to a tighter bound? As we show in the next subsection, the answer is yes.
%$$\textup{Var} \left( \sum_{k=1}^K w_k \right)  = \mathbb{E}[\left(\sum_{k=1}^K w_k \right)^2] - K^2 \mathbb{E} [w_1]  $$ 
\subsection{The more negatively dependent the weights, the tighter the bound}
An important example of supermodular function is the following: let $\phi$ be a convex function and $\boldsymbol{\alpha}$ a vector with non-negative coefficients, then  $\mathbf{w} \mapsto \phi \left(\boldsymbol{\alpha}^T\mathbf{w} \right)$ is supermodular. Using this fact with $\phi = -\log$ immediately immediately leads to
\begin{equation}
	\label{eq:smcx}
	\mathbf{w}\preceq_\textup{SM} \mathbf{v} \implies \boldsymbol{\alpha}^T\mathbf{w} \preceq_\textup{CX} \boldsymbol{\alpha}^T\mathbf{v},
\end{equation}
and to the following monotonicity theorem.
\begin{theorem}[\textbf{negative dependence tightens the bound}] For all pairs $Q_1, Q_2$ of probability distributions over $\mathbb{R}^K$,
	\begin{equation}
		Q_1\preceq_\textup{SM} Q_2 \implies \mathcal{L}_{\boldsymbol{\alpha}} (Q_1)\geq  \mathcal{L}_{\boldsymbol{\alpha}} (Q_2).
	\end{equation}
\end{theorem}

In other words, \emph{the lower bound gets tighter when the weights get more negatively dependent} (in the supermodular sense).
This gives a theoretical support to the successful recent applications of negative dependence to tighten variational bounds. Beyond bound tightening, combining Equation \eqref{eq:smcx} and Example \ref{ex:div} ensures that more negative dependence will also provide more accurate likelihood estimators according to any $f$-divergence.

The main limitation of our result is that it is difficult to control the supermodular order in practice. A silver lining to this is the central role played by the supermodular order among dependence measures. In particular, the popular notion of \emph{negative association} is in a sense stronger than the supermodular order (for a more general result than the simple one from \citealp{christofides2004}, cited above, see \citealp{shaked2007}, Theorem 9.E.8). For instance, when $K=2$,
\begin{equation}
	\textup{Cov}_{Q_1}(h_1(w_1),h_2(w_2)) \leq \textup{Cov}_{Q_2}(h_1(v_1),h_2(v_2))
\end{equation}
for all increasing functions $h_1$ and $h_2$ implies that $Q_1 \preceq_\textup{SM} Q_2$.

\section{A non-uniform generalisation of Burda's result via majorisation}

In this section, we wish to extend sample monotonicty to the case of non-uniform bounds  $\mathcal{L}_{\boldsymbol{\alpha}}(Q)$. To this end, we use the concept of \emph{majorisation} that was popularised in the influential book of \citet[Chapter 2]{hardy1952inequalities}. For a good review of majorisation and its applications, see \citet{marshall2011inequalities}.

\begin{definition}[\textbf{majorisation}]
	Let $\boldsymbol{\alpha},\boldsymbol{\beta} \in \Delta_K$. We say that $\boldsymbol{\alpha}$ \textbf{majorises} $\boldsymbol{\beta}$ if $\boldsymbol{\beta}$  is in the convex hull of all vectors obtained by permuting the coordinates of  $\boldsymbol{\alpha}$. We denote $\boldsymbol{\alpha}  \preceq_\textup{M} \boldsymbol{\beta}$. This is equivalent to the condition
	\begin{equation}
		\forall k \in \{1,...,K\}, \;	\sum_{j=1}^k \alpha^{(j)}\leq \sum_{j=1}^k \beta^{(j)},
	\end{equation}
	where  $({\alpha}^{(1)},...,{\alpha}^{(K)})$ and  $({\beta}^{(1)},...,{\beta}^{(K)})$ are reordered versions of $\boldsymbol{\alpha}$ and $\boldsymbol{\beta}$, sorted in decreasing order.
\end{definition}
Roughly speaking, $\boldsymbol{\alpha}  \preceq_\textup{M} \boldsymbol{\beta}$ when the coefficients of $\boldsymbol{\beta}$ are ``more spread out" than those of $\boldsymbol{\alpha} $. Indeed, for example, we have for all $\boldsymbol{\alpha} \in \Delta_K$,
\begin{equation}
	\label{eq:spread_out}
	(1/K,...,1/K) \preceq_\textup{M} \boldsymbol{\alpha} \preceq_\textup{M} (1,0,...,0).
\end{equation}
Other examples of majorisation can be found in \citet{marshall2011inequalities}, notably in Chapters 1 and 5.

We can now state and prove the more general version of sample monotonicity of \citet{marshall1965inequality}. Note that our proof is quite different, but not particularly original. Our proof is similar in spirit to a simple proof of the result that any convex and symmetric function is Schur-convex, i.e. respects the majorisation order (see e.g. the proof of \citealp{marshall2011inequalities}, Proposition C.3).

\begin{theorem}[\citealp{marshall1965inequality}]
	\label{th:marshall}
	Let $\boldsymbol{\alpha},\boldsymbol{\beta} \in \Delta_K$. If the weights are exchangeable, then
	\begin{equation}
		\boldsymbol{\alpha}\preceq_\textup{M} \boldsymbol{\beta} \implies \boldsymbol{\alpha}^T\mathbf{w} \preceq_\textup{CX} \boldsymbol{\beta}^T\mathbf{w}.
	\end{equation}
\end{theorem}
\begin{proof} Let $\phi$ be a convex function.
	Exchangeabity of the weights and convexity of $\phi$ imply that the function $f: \boldsymbol{x} \mapsto \mathbb{E}_Q[\phi(\boldsymbol{x}^T\mathbf{w})] $ is convex and symmetric (i.e.~permutation-invariant).
	Consider now $\boldsymbol{\alpha}  \preceq_\textup{M} \boldsymbol{\beta}$. Then, $\boldsymbol{\alpha} $ belongs to the 
	convex hull of all vectors obtained by permuting the coordinates of  $\boldsymbol{\alpha}$. Since $f$ is symmetric, all these permuted vectors lead to the same value of $f$. Convexity of $f$ then leads to $\mathbb{E}_Q[\phi(\boldsymbol{\alpha}^T\mathbf{w})] \leq \mathbb{E}_Q[\phi(\boldsymbol{\beta}^T\mathbf{w})]$.
\end{proof}
This immediately gives a non-uniform version of sample monotonicity, that can be interpreted this way: \emph{the more spread out the coefficients $\alpha_{1},...,{\alpha}_{K}$ of the weights, the tighter the bound}.

\begin{cor}[\textbf{non-uniform sample monotonicity}]
	\label{cor:non-unif}
	Let $\boldsymbol{\alpha},\boldsymbol{\beta} \in \Delta_K$. If the weights are exchangeable, then
	\begin{equation}
		\boldsymbol{\alpha}  \preceq_\textup{M} \boldsymbol{\beta} \implies  \mathcal{L}_{\boldsymbol{\alpha}}(Q) \geq \mathcal{L}_{\boldsymbol{\beta}}(Q).
	\end{equation}
\end{cor}

There are several immediate corollaries of Corollary \ref{cor:non-unif}. The first one is the original sample monotonicity result of \citet{burda2016}, stated as Theorem \ref{th:sample_mono} in our paper . Indeed, this result is direct consequence of the fact that $(1/(K+1),...,1/(K+1))  \preceq_\textup{M} (1/K,...,1/K,0) $. More generally, the left hand of Equation \eqref{eq:spread_out} implies that, when the importance weights are exchangeable,
\begin{equation}
	\mathcal{L}_{\boldsymbol{\alpha}}(Q) \leq  \mathcal{L}_K(Q),
\end{equation}
which means that \emph{if the weights are exchangeable, it is optimal to use the standard uniform average}. The practical guideline that comes with this is that we should not bother learning non-uniform coefficients when the weights are exchangeable. This is quite in line with the reasoning of \citet{huang2019}, who advocated the use of non-uniform coefficients together with non-exchangeable importance weights.

Less importantly, another corollary is Example \ref{eq:exch_necessary}. Indeed, using under the setup of Example  \ref{eq:exch_necessary}, the fact that $(1/2,1/2) \preceq_\textup{M} (2/3,1/3)$ implies that
\begin{equation}
	\label{eq:proof_exch_necessary}
	\mathcal{L}_{2}(Q) = \mathbb{E}\left[\frac{x+y}{2}\right] \geq \mathbb{E}\left[\frac{2x}{3} + \frac{y}{3}\right] = \mathbb{E}\left[\frac{x+y+x}{3}\right]= \mathcal{L}_{3}(Q).
\end{equation}

Another consequence of Theorem \ref{th:marshall} and Example \ref{ex:div} is that the $f$-divergence between the likelihood and its importance-sampling approximation is a decreasing function of $K$.

\section{Conclusion}

We have presented several simple results inspired by the sample monotonicity theorem of \citet{burda2016}. Further refinements and generalisations appear possible. For instance, in the non-exchangeable case, it seems reasonable that using non-uniform coefficients can be optimal, but our paper does not offer a theory for this.

Another important question concerns additional applications of such results. Concerning negative dependence, an interesting question is whether or not these sorts of investigations could provide a guide to design proposal distributions with the ``right amount of correlation'' required to tighten bounds.

As a concluding note, let us mention that we were surprised to notice that stochastic orders have been seldom applied to studying Monte Carlo methods. Among the few papers that we found that explored this connection,  a nice line of work originated by \citet{andrieu2016establishing} has used the convex order to analyse Markov chain Monte Carlo algorithms \citep{bornn2017use,leskela2017conditional,andrieu2018theoretical}. Other interesting work using stochastic order in a Monte Carlo setting include \citet{goldstein2011stochastic, goldstein2012stochastic} and \citet{bernard2019estimating}. We believe that the convex order provides a quite compelling way of assessing the accuracy of different Monte Carlo approximations, and can be a very valuable tool within the Monte Carlo theoretical toolbox. 

\begin{acks}[Acknowledgments]

This work has been supported by the French government, through the 3IA Côte d’Azur Investments in the Future project managed by the National Research Agency (ANR) with the reference number ANR-19-P3IA-0002. Furthermore, it was supported by the Novo Nordisk Foundation (NNF20OC0062606 and NNF20OC0065611) and the Independent Research Fund Denmark (9131-00082B).

 \end{acks}

\appendix
\section*{Appendix A. Proofs of examples}

\paragraph{Example \ref{prop:successes}} \emph{Let $R$ and $R'$ be either two gamma, two inverse gamma, or two log-normal distributions with finite and equal means and finite variances. Then
	\begin{equation}
		%\label{eq:successes}
		\textup{Var}(R) < \textup{Var}[R'] \iff \textup{Var}(\log R) < \textup{Var}[\log R'] \iff  \mathbb{E}[\log R] > \mathbb{E}[\log R'].
\end{equation}}

\begin{proof} We will treat each family separately.
	\paragraph{The gamma case.} Let $R \sim \mathcal{G} (\alpha, \beta)$ and $R' \sim \mathcal{G} (a, b)$ with $\alpha, \beta ,a,b> 0$. We remind that 
	\begin{equation}
		\mathbb{E}[R] = \frac{\alpha}{\beta}, \; \mathbb{E}[\log R] = \psi(\alpha)- \log(\beta), \; \textup{Var}(\log R) = \psi^{(1)}(\alpha), \; \text{Var}(R) = \frac{\alpha}{\beta^2},
	\end{equation}
	where $\psi$ is the digamma function and $\psi^{(1)}$ is the trigamma function.
	Assuming that $R$ and $R'$ have the same mean leads to $a= \alpha b / \beta$. We have then:
	\begin{align}
		\textup{Var}(R) < \textup{Var}[R'] \iff \frac{\alpha}{\beta^2} < \frac{a}{b^2} = \frac{\alpha}{\beta b} \iff \beta > b, 
	\end{align}
	%	To control the bounds, we will use the following useful property of the digamma function \citep[Equation 2.1]{alzer1997}:
	%	\begin{equation}
	%\label{eq:alzer}
	%\log x - \frac{1}{x} \leq \psi (x) \leq \log x - \frac{1}{2x}.
	%\end{equation}
	and, using the fact that $\psi^{(1)}$ decreases,
	\begin{equation}
		\textup{Var}(\log R) < \textup{Var}[\log R'] \iff \psi^{(1)}(\alpha) < \psi^{(1)}(a) \iff \alpha>a \iff \beta > b.
	\end{equation}
	Regarding the bounds, we have
	\begin{align}
		\mathbb{E}[\log R] - \mathbb{E}[\log R'] &= \psi(\alpha) - \psi\left(\frac{\alpha b } {\beta}\right) + \log\left(\frac{b}{\beta}\right) \\
		&= \psi(\alpha) - \log(\alpha) - \left(\psi\left(\frac{\alpha b } {\beta}\right) - \log\left(\frac{\alpha b } {\beta}\right)\right).
	\end{align}
	Therefore, since the function $\psi - \log$ is increasing (see e.g.~\citealp{alzer1997}, Theorem 1), we get
	\begin{equation}
		\mathbb{E}[\log R] > \mathbb{E}[\log R'] \iff \alpha > \frac{\alpha b}{\beta} \iff \beta> b.
	\end{equation}

	\paragraph{The inverse gamma case.} Let $R \sim \mathcal{IG} (\alpha, \beta)$ and $R' \sim \mathcal{IG} (a, b)$ with $\alpha, \beta ,a,b> 0$. Since we assume that $R$ and $R'$ have finite variance, we must have $\alpha,a>2$. We remind that 
	\begin{equation}
		\mathbb{E}[R] = \frac{\beta}{\alpha-1}, \; \mathbb{E}[\log R] =  \log(\beta) - \psi(\alpha), \; \textup{Var}(\log R) = \psi^{(1)} (\alpha), \; \text{Var}(R) = \frac{\beta}{(\alpha-1)^2}.
	\end{equation}
	Assuming equality of the means leads to $b = \beta (a-1)/(\alpha - 1)$. We have then:
	\begin{equation}
		\textup{Var}(R) < \textup{Var}[R'] \iff \alpha > a.
	\end{equation}
	Since the trigamma function decreases, we also have
	\begin{equation}
		\textup{Var}(\log R) < \textup{Var}[\log R'] \iff \alpha > a.
	\end{equation}
	Let us now look at the bounds. We have
	\begin{align}
		\mathbb{E}[\log R] - \mathbb{E}[\log R'] &= \psi(a) - \psi(\alpha) +\log \left(\frac{\alpha - 1}{a -1} \right) \\
		&= \psi(a) - \log(a) - \left(\psi(\alpha) - \log(\alpha)\right) + \log\left(\frac{\alpha-1}{\alpha}\right) - \log \left(\frac{a-1}{a}\right).
	\end{align}
	Since the functions $\psi - \log$ and $x\mapsto \log((x-1)/x)$ are  increasing, we get
	\begin{equation}
		\mathbb{E}[\log R] > \mathbb{E}[\log R'] \iff \alpha > a.
	\end{equation}

	\paragraph{The lognormal case.}
	Let $R \sim \text{log} \,\mathcal{N} (\mu,\sigma)$ and $R' \sim \text{log} \,\mathcal{N} (m,s)$ with $\mu, m \in \mathbb{R}$ and $\sigma, s > 0$. We remind that
	\begin{equation}
		\mathbb{E}[R] = e^{\mu + \sigma^2/2}, \; \mathbb{E}[\log R] = \mu, \; \textup{Var}(\log R) =\sigma^2, \; \text{Var}(R)= \left(e^{\sigma^2} - 1 \right) e^{2\mu + \sigma^2}.
	\end{equation}
	The equality of the means implies that $2\mu + \sigma^2 = 2m + s^2$. Therefore, we have
	\begin{align}
		\textup{Var}(R) < \textup{Var}[R'] &\iff \sigma^2 < s^2 \iff \textup{Var}(\log R) < \textup{Var}(\log R') \\ &\iff  \sigma^2 < 2 \mu + \sigma^2 - 2 m \\  &\iff m  <  \mu  \iff  \mathbb{E}[\log R']< \mathbb{E}[\log R].
\end{align} 	\end{proof}

\paragraph{Example \ref{ex:failures}}\emph{
	Let $R$ be an inverse-gamma variable with finite mean. It is possible to find a log-normal random variable $R'$ such that
	\begin{itemize}
		\item $\mathbb{E}[R] = \mathbb{E}[R']$, $\textup{Var}(R) = \infty$, $\textup{Var}[R'] < \infty$,
		\item $\mathbb{E}[\log R] > \mathbb{E}[\log R']$.
\end{itemize}}

\begin{proof}
	Let $R \sim \mathcal{IG} (\alpha,\beta)$ and  $R' \sim \text{log} \,\mathcal{N} (\mu,\sigma)$ with $\alpha \in (1,2)$, $\beta,\sigma>0$ and $\mu \in \mathbb{R}$. Since we chose $\alpha \in (1,2)$, the mean of $R$ will be finite but its variance will be infinite. To get equality of the means of $R$ and $R'$, we further assume that
	\begin{equation}
		\mu = \log(\beta) - \log(\alpha-1)- \frac{\sigma^2}{2}.
	\end{equation}
	The difference between the bounds is then equal to
	\begin{equation}
		\mathbb{E}[\log R] - \mathbb{E}[\log R']  = - \psi(\alpha) + \log (\beta) - \mu =  - \psi(\alpha) + \log (\beta) + \frac{\sigma^2}{2},
	\end{equation}
	which can be arbitrarily large provided that $\sigma$ is large enough.
\end{proof}

\section*{Appendix B. Experimental details on the variance experiment}

The architecture of the DLVM that we trained is similar to the one of \citet{burda2016} except that we use a Student's $t$ distribution for the proposal instead of a Gaussian, following \citet{domke2018}. This choice was made because using a heavy-tailed proposal is likely to lead to better-behaved importance weights. The model is trained for $1{,}500$ epochs with a batch size of $20$ and the Adam optimiser of \citet{kingma2014adam}, with learning rate $2.10
^{-4}$.

%%%%%%%%%%%%%%%%%%%%%%%%%%%%%%%%%%%%%%%%%%%%%%
%% Supplementary Material, if any, should   %%
%% be provided in {supplement} environment  %%
%% with title and short description.        %%
%%%%%%%%%%%%%%%%%%%%%%%%%%%%%%%%%%%%%%%%%%%%%%
%\begin{supplement}
%\stitle{???}
%\sdescription{???.}
%\end{supplement}

%% if your bibliography is in bibtex format, uncomment commands:
%\bibliographystyle{imsart-nameyear} % Style BST file (imsart-number.bst or imsart-nameyear.bst)
\bibliography{biblio}

\begin{thebibliography}{67}
% BibTex style file: imsart-nameyear.bst, 2017-11-03
% Default style options (sort=1,type=nameyear).
% Used options (sort=1,type=nameyear).

\bibitem[\protect\citeauthoryear{Alemi et~al.}{2017}]{alemi2017}
\begin{binproceedings}[author]
\bauthor{\bsnm{Alemi},~\bfnm{A.}\binits{A.}},
  \bauthor{\bsnm{Fischer},~\bfnm{I.}\binits{I.}},
  \bauthor{\bsnm{Dillon},~\bfnm{J.}\binits{J.}} \AND
  \bauthor{\bsnm{Murphy},~\bfnm{K.}\binits{K.}}
(\byear{2017}).
\btitle{Deep Variational Information Bottleneck}.
In \bbooktitle{International Conference on Learning Representations}.
\end{binproceedings}
\endbibitem

\bibitem[\protect\citeauthoryear{Alzer}{1997}]{alzer1997}
\begin{barticle}[author]
\bauthor{\bsnm{Alzer},~\bfnm{H.}\binits{H.}}
(\byear{1997}).
\btitle{On some inequalities for the gamma and psi functions}.
\bjournal{Mathematics of computation}
\bvolume{66}
\bpages{373--389}.
\end{barticle}
\endbibitem

\bibitem[\protect\citeauthoryear{Andrieu, Lee and
  Vihola}{2018}]{andrieu2018theoretical}
\begin{bincollection}[author]
\bauthor{\bsnm{Andrieu},~\bfnm{C.}\binits{C.}},
  \bauthor{\bsnm{Lee},~\bfnm{A.}\binits{A.}} \AND
  \bauthor{\bsnm{Vihola},~\bfnm{M.}\binits{M.}}
(\byear{2018}).
\btitle{Theoretical and Methodological Aspects of{ Markov Chain Monte Carlo}
  Computations with Noisy Likelihoods}.
In \bbooktitle{Handbook of Approximate Bayesian Computation}
\bpages{243--268}.
\bpublisher{Chapman and Hall/CRC}.
\end{bincollection}
\endbibitem

\bibitem[\protect\citeauthoryear{Andrieu and
  Vihola}{2016}]{andrieu2016establishing}
\begin{barticle}[author]
\bauthor{\bsnm{Andrieu},~\bfnm{C.}\binits{C.}} \AND
  \bauthor{\bsnm{Vihola},~\bfnm{M.}\binits{M.}}
(\byear{2016}).
\btitle{Establishing some order amongst exact approximations of {MCMCs}}.
\bjournal{The Annals of Applied Probability}
\bvolume{26}
\bpages{2661--2696}.
\end{barticle}
\endbibitem

\bibitem[\protect\citeauthoryear{Ba et~al.}{2015}]{ba2015learning}
\begin{binproceedings}[author]
\bauthor{\bsnm{Ba},~\bfnm{J.}\binits{J.}},
  \bauthor{\bsnm{Salakhutdinov},~\bfnm{R.~R.}\binits{R.~R.}},
  \bauthor{\bsnm{Grosse},~\bfnm{R.~B.}\binits{R.~B.}} \AND
  \bauthor{\bsnm{Frey},~\bfnm{B.~J.}\binits{B.~J.}}
(\byear{2015}).
\btitle{Learning wake-sleep recurrent attention models}.
In \bbooktitle{Advances in Neural Information Processing Systems}
\bpages{2593--2601}.
\end{binproceedings}
\endbibitem

\bibitem[\protect\citeauthoryear{Bardenet and Hardy}{2020}]{bardenet2020monte}
\begin{barticle}[author]
\bauthor{\bsnm{Bardenet},~\bfnm{R.}\binits{R.}} \AND
  \bauthor{\bsnm{Hardy},~\bfnm{A.}\binits{A.}}
(\byear{2020}).
\btitle{{Monte Carlo} with determinantal point processes}.
\bjournal{The Annals of Applied Probability}
\bvolume{30}
\bpages{368--417}.
\end{barticle}
\endbibitem

\bibitem[\protect\citeauthoryear{Bernard and
  Leduc}{2019}]{bernard2019estimating}
\begin{barticle}[author]
\bauthor{\bsnm{Bernard},~\bfnm{L.}\binits{L.}} \AND
  \bauthor{\bsnm{Leduc},~\bfnm{P.}\binits{P.}}
(\byear{2019}).
\btitle{Estimating a probability of failure with the convex order in computer
  experiments}.
\bjournal{arXiv preprint arXiv:1907.01781}.
\end{barticle}
\endbibitem

\bibitem[\protect\citeauthoryear{Bornn et~al.}{2017}]{bornn2017use}
\begin{barticle}[author]
\bauthor{\bsnm{Bornn},~\bfnm{L.}\binits{L.}},
  \bauthor{\bsnm{Pillai},~\bfnm{N.~S.}\binits{N.~S.}},
  \bauthor{\bsnm{Smith},~\bfnm{A.}\binits{A.}} \AND
  \bauthor{\bsnm{Woodard},~\bfnm{D.}\binits{D.}}
(\byear{2017}).
\btitle{The use of a single pseudo-sample in approximate {Bayesian}
  computation}.
\bjournal{Statistics and Computing}
\bvolume{27}
\bpages{583--590}.
\end{barticle}
\endbibitem

\bibitem[\protect\citeauthoryear{Bornschein and
  Bengio}{2015}]{bornschein2014reweighted}
\begin{binproceedings}[author]
\bauthor{\bsnm{Bornschein},~\bfnm{J.}\binits{J.}} \AND
  \bauthor{\bsnm{Bengio},~\bfnm{Y.}\binits{Y.}}
(\byear{2015}).
\btitle{Reweighted wake-sleep}.
In \bbooktitle{International Conference on Learning Representations}.
\end{binproceedings}
\endbibitem

\bibitem[\protect\citeauthoryear{Bottou, Curtis and Nocedal}{2018}]{bottou2018}
\begin{barticle}[author]
\bauthor{\bsnm{Bottou},~\bfnm{L.}\binits{L.}},
  \bauthor{\bsnm{Curtis},~\bfnm{F.~E.}\binits{F.~E.}} \AND
  \bauthor{\bsnm{Nocedal},~\bfnm{J.}\binits{J.}}
(\byear{2018}).
\btitle{Optimization methods for large-scale machine learning}.
\bjournal{{SIAM} Review}
\bvolume{60}
\bpages{223--311}.
\end{barticle}
\endbibitem

\bibitem[\protect\citeauthoryear{Burda, Grosse and
  Salakhutdinov}{2016}]{burda2016}
\begin{binproceedings}[author]
\bauthor{\bsnm{Burda},~\bfnm{Y.}\binits{Y.}},
  \bauthor{\bsnm{Grosse},~\bfnm{R.}\binits{R.}} \AND
  \bauthor{\bsnm{Salakhutdinov},~\bfnm{R.}\binits{R.}}
(\byear{2016}).
\btitle{Importance weighted autoencoders}.
In \bbooktitle{International Conference on Learning Representations}.
\end{binproceedings}
\endbibitem

\bibitem[\protect\citeauthoryear{Carr and Wu}{2003}]{carr2003finite}
\begin{barticle}[author]
\bauthor{\bsnm{Carr},~\bfnm{P.}\binits{P.}} \AND
  \bauthor{\bsnm{Wu},~\bfnm{L.}\binits{L.}}
(\byear{2003}).
\btitle{The finite moment log stable process and option pricing}.
\bjournal{The Journal of Finance}
\bvolume{58}
\bpages{753--777}.
\end{barticle}
\endbibitem

\bibitem[\protect\citeauthoryear{Chatterjee and
  Diaconis}{2018}]{chatterjee2018}
\begin{barticle}[author]
\bauthor{\bsnm{Chatterjee},~\bfnm{S.}\binits{S.}} \AND
  \bauthor{\bsnm{Diaconis},~\bfnm{P.}\binits{P.}}
(\byear{2018}).
\btitle{The sample size required in importance sampling}.
\bjournal{The Annals of Applied Probability}
\bvolume{28}
\bpages{1099--1135}.
\end{barticle}
\endbibitem

\bibitem[\protect\citeauthoryear{Che et~al.}{2020}]{che2020}
\begin{barticle}[author]
\bauthor{\bsnm{Che},~\bfnm{T.}\binits{T.}},
  \bauthor{\bsnm{Liu},~\bfnm{X.}\binits{X.}},
  \bauthor{\bsnm{Li},~\bfnm{S.}\binits{S.}},
  \bauthor{\bsnm{Ge},~\bfnm{Y.}\binits{Y.}},
  \bauthor{\bsnm{Zhang},~\bfnm{R.}\binits{R.}},
  \bauthor{\bsnm{Xiong},~\bfnm{C.}\binits{C.}} \AND
  \bauthor{\bsnm{Bengio},~\bfnm{Y.}\binits{Y.}}
(\byear{2020}).
\btitle{Deep verifier networks: Verification of deep discriminative models with
  deep generative models}.
\bjournal{arXiv preprint arXiv:1911.07421}.
\end{barticle}
\endbibitem

\bibitem[\protect\citeauthoryear{Christofides and
  Vaggelatou}{2004}]{christofides2004}
\begin{barticle}[author]
\bauthor{\bsnm{Christofides},~\bfnm{T.~C.}\binits{T.~C.}} \AND
  \bauthor{\bsnm{Vaggelatou},~\bfnm{E.}\binits{E.}}
(\byear{2004}).
\btitle{A connection between supermodular ordering and positive/negative
  association}.
\bjournal{Journal of Multivariate analysis}
\bvolume{88}
\bpages{138--151}.
\end{barticle}
\endbibitem

\bibitem[\protect\citeauthoryear{Dhekane}{2020}]{dhekane2021improving}
\begin{bmastersthesis}[author]
\bauthor{\bsnm{Dhekane},~\bfnm{E.~G.}\binits{E.~G.}}
(\byear{2020}).
\btitle{On improving variational inference with low-variance multi-sample
  estimators},
\btype{Master's thesis},
\bpublisher{Université de Montréal}.
\end{bmastersthesis}
\endbibitem

\bibitem[\protect\citeauthoryear{Dieng and Paisley}{2019}]{dieng2019reweighted}
\begin{barticle}[author]
\bauthor{\bsnm{Dieng},~\bfnm{A.~B.}\binits{A.~B.}} \AND
  \bauthor{\bsnm{Paisley},~\bfnm{J.}\binits{J.}}
(\byear{2019}).
\btitle{Reweighted expectation maximization}.
\bjournal{arXiv preprint arXiv:1906.05850}.
\end{barticle}
\endbibitem

\bibitem[\protect\citeauthoryear{Domke and Sheldon}{2018}]{domke2018}
\begin{binproceedings}[author]
\bauthor{\bsnm{Domke},~\bfnm{J.}\binits{J.}} \AND
  \bauthor{\bsnm{Sheldon},~\bfnm{D.~R.}\binits{D.~R.}}
(\byear{2018}).
\btitle{Importance weighting and variational inference}.
In \bbooktitle{Advances in neural information processing systems}
\bpages{4470--4479}.
\end{binproceedings}
\endbibitem

\bibitem[\protect\citeauthoryear{Domke and Sheldon}{2019}]{domke2019}
\begin{binproceedings}[author]
\bauthor{\bsnm{Domke},~\bfnm{J.}\binits{J.}} \AND
  \bauthor{\bsnm{Sheldon},~\bfnm{D.~R.}\binits{D.~R.}}
(\byear{2019}).
\btitle{Divide and Couple: Using Monte Carlo Variational Objectives for
  Posterior Approximation}.
In \bbooktitle{Advances in neural information processing systems}
\bpages{338--347}.
\end{binproceedings}
\endbibitem

\bibitem[\protect\citeauthoryear{Elvira et~al.}{2019}]{elvira2019generalized}
\begin{barticle}[author]
\bauthor{\bsnm{Elvira},~\bfnm{V.}\binits{V.}},
  \bauthor{\bsnm{Martino},~\bfnm{L.}\binits{L.}},
  \bauthor{\bsnm{Luengo},~\bfnm{D.}\binits{D.}} \AND
  \bauthor{\bsnm{Bugallo},~\bfnm{M.~F.}\binits{M.~F.}}
(\byear{2019}).
\btitle{Generalized multiple importance sampling}.
\bjournal{Statistical Science}
\bvolume{34}
\bpages{129--155}.
\end{barticle}
\endbibitem

\bibitem[\protect\citeauthoryear{Finke and Thiery}{2019}]{finke2019importance}
\begin{barticle}[author]
\bauthor{\bsnm{Finke},~\bfnm{A.}\binits{A.}} \AND
  \bauthor{\bsnm{Thiery},~\bfnm{A.~H.}\binits{A.~H.}}
(\byear{2019}).
\btitle{On importance-weighted autoencoders}.
\bjournal{arXiv preprint arXiv:1907.10477}.
\end{barticle}
\endbibitem

\bibitem[\protect\citeauthoryear{Geweke}{1989}]{geweke1989bayesian}
\begin{barticle}[author]
\bauthor{\bsnm{Geweke},~\bfnm{J.}\binits{J.}}
(\byear{1989}).
\btitle{Bayesian inference in econometric models using {Monte Carlo}
  integration}.
\bjournal{Econometrica: Journal of the Econometric Society}
\bpages{1317--1339}.
\end{barticle}
\endbibitem

\bibitem[\protect\citeauthoryear{Geyer}{1994}]{geyer1994convergence}
\begin{barticle}[author]
\bauthor{\bsnm{Geyer},~\bfnm{C.~J.}\binits{C.~J.}}
(\byear{1994}).
\btitle{On the convergence of Monte Carlo maximum likelihood calculations}.
\bjournal{Journal of the Royal Statistical Society: Series B (Methodological)}
\bvolume{56}
\bpages{261--274}.
\end{barticle}
\endbibitem

\bibitem[\protect\citeauthoryear{Goldstein, Rinott and
  Scarsini}{2011}]{goldstein2011stochastic}
\begin{barticle}[author]
\bauthor{\bsnm{Goldstein},~\bfnm{L.}\binits{L.}},
  \bauthor{\bsnm{Rinott},~\bfnm{Y.}\binits{Y.}} \AND
  \bauthor{\bsnm{Scarsini},~\bfnm{M.}\binits{M.}}
(\byear{2011}).
\btitle{Stochastic comparisons of stratified sampling techniques for some
  {Monte Carlo} estimators}.
\bjournal{Bernoulli}
\bvolume{17}
\bpages{592--608}.
\end{barticle}
\endbibitem

\bibitem[\protect\citeauthoryear{Goldstein, Rinott and
  Scarsini}{2012}]{goldstein2012stochastic}
\begin{barticle}[author]
\bauthor{\bsnm{Goldstein},~\bfnm{L.}\binits{L.}},
  \bauthor{\bsnm{Rinott},~\bfnm{Y.}\binits{Y.}} \AND
  \bauthor{\bsnm{Scarsini},~\bfnm{M.}\binits{M.}}
(\byear{2012}).
\btitle{Stochastic comparisons of symmetric sampling designs}.
\bjournal{Methodology and Computing in Applied Probability}
\bvolume{14}
\bpages{407--420}.
\end{barticle}
\endbibitem

\bibitem[\protect\citeauthoryear{Hardy, Littlewood and
  P{\'o}lya}{1952}]{hardy1952inequalities}
\begin{bbook}[author]
\bauthor{\bsnm{Hardy},~\bfnm{G.~H.}\binits{G.~H.}},
  \bauthor{\bsnm{Littlewood},~\bfnm{J.~E.}\binits{J.~E.}} \AND
  \bauthor{\bsnm{P{\'o}lya},~\bfnm{G.}\binits{G.}}
(\byear{1952}).
\btitle{Inequalities (2nd edition)}.
\bpublisher{Cambridge university press}.
\end{bbook}
\endbibitem

\bibitem[\protect\citeauthoryear{Hoogeboom, Cohen and
  Tomczak}{2020}]{hoogeboom2020learning}
\begin{barticle}[author]
\bauthor{\bsnm{Hoogeboom},~\bfnm{E.}\binits{E.}},
  \bauthor{\bsnm{Cohen},~\bfnm{T.~S.}\binits{T.~S.}} \AND
  \bauthor{\bsnm{Tomczak},~\bfnm{J.~M.}\binits{J.~M.}}
(\byear{2020}).
\btitle{Learning Discrete Distributions by Dequantization}.
\bjournal{arXiv preprint arXiv:2001.11235}.
\end{barticle}
\endbibitem

\bibitem[\protect\citeauthoryear{Huang and Courville}{2019}]{huang2019note}
\begin{barticle}[author]
\bauthor{\bsnm{Huang},~\bfnm{C.~W.}\binits{C.~W.}} \AND
  \bauthor{\bsnm{Courville},~\bfnm{A.}\binits{A.}}
(\byear{2019}).
\btitle{Note on the bias and variance of variational inference}.
\bjournal{arXiv preprint arXiv:1906.03708}.
\end{barticle}
\endbibitem

\bibitem[\protect\citeauthoryear{Huang et~al.}{2019}]{huang2019}
\begin{binproceedings}[author]
\bauthor{\bsnm{Huang},~\bfnm{C.~W.}\binits{C.~W.}},
  \bauthor{\bsnm{Sankaran},~\bfnm{K.}\binits{K.}},
  \bauthor{\bsnm{Dhekane},~\bfnm{E.}\binits{E.}},
  \bauthor{\bsnm{Lacoste},~\bfnm{A.}\binits{A.}} \AND
  \bauthor{\bsnm{Courville},~\bfnm{A.}\binits{A.}}
(\byear{2019}).
\btitle{Hierarchical Importance Weighted Autoencoders}.
In \bbooktitle{International Conference on Machine Learning}
\bpages{2869--2878}.
\end{binproceedings}
\endbibitem

\bibitem[\protect\citeauthoryear{Ipsen, Mattei and
  Frellsen}{2021}]{ipsen2021not}
\begin{binproceedings}[author]
\bauthor{\bsnm{Ipsen},~\bfnm{N.~B.}\binits{N.~B.}},
  \bauthor{\bsnm{Mattei},~\bfnm{P.~A.}\binits{P.~A.}} \AND
  \bauthor{\bsnm{Frellsen},~\bfnm{J.}\binits{J.}}
(\byear{2021}).
\btitle{{not-MIWAE}: Deep Generative Modelling with Missing not at Random
  Data}.
In \bbooktitle{International Conference on Learning Representations}.
\end{binproceedings}
\endbibitem

\bibitem[\protect\citeauthoryear{Jensen}{1906}]{jensen1906}
\begin{barticle}[author]
\bauthor{\bsnm{Jensen},~\bfnm{J.}\binits{J.}}
(\byear{1906}).
\btitle{Sur les fonctions convexes et les in{\'e}galit{\'e}s entre les valeurs
  moyennes}.
\bjournal{Acta mathematica}
\bvolume{30}
\bpages{175--193}.
\end{barticle}
\endbibitem

\bibitem[\protect\citeauthoryear{Joe}{1997}]{joe1997}
\begin{bbook}[author]
\bauthor{\bsnm{Joe},~\bfnm{H.}\binits{H.}}
(\byear{1997}).
\btitle{Multivariate models and multivariate dependence concepts}.
\bpublisher{CRC Press}.
\end{bbook}
\endbibitem

\bibitem[\protect\citeauthoryear{Josse, Mayer and Vert}{2020}]{josse2020}
\begin{barticle}[author]
\bauthor{\bsnm{Josse},~\bfnm{J.}\binits{J.}},
  \bauthor{\bsnm{Mayer},~\bfnm{I.}\binits{I.}} \AND
  \bauthor{\bsnm{Vert},~\bfnm{J.~P.}\binits{J.~P.}}
(\byear{2020}).
\btitle{{MissDeepCausal}: causal inference from incomplete data using deep
  latent variable models}.
\bjournal{Openreview preprint}.
\end{barticle}
\endbibitem

\bibitem[\protect\citeauthoryear{Kim, Hwang and Kim}{2020}]{kim2020}
\begin{binproceedings}[author]
\bauthor{\bsnm{Kim},~\bfnm{D.}\binits{D.}},
  \bauthor{\bsnm{Hwang},~\bfnm{J.}\binits{J.}} \AND
  \bauthor{\bsnm{Kim},~\bfnm{Y.}\binits{Y.}}
(\byear{2020}).
\btitle{On casting importance weighted autoencoder to an {EM} algorithm to
  learn deep generative models}.
In \bbooktitle{Proceedings of the Twenty Third International Conference on
  Artificial Intelligence and Statistics}
\bvolume{108}
\bpages{2153--2163}.
\end{binproceedings}
\endbibitem

\bibitem[\protect\citeauthoryear{Kingma and Ba}{2014}]{kingma2014adam}
\begin{barticle}[author]
\bauthor{\bsnm{Kingma},~\bfnm{D.~P.}\binits{D.~P.}} \AND
  \bauthor{\bsnm{Ba},~\bfnm{J.}\binits{J.}}
(\byear{2014}).
\btitle{Adam: A method for stochastic optimization}.
\bjournal{Proceedings of the International Conference on Learning
  Representations}.
\end{barticle}
\endbibitem

\bibitem[\protect\citeauthoryear{Kingma and Welling}{2014}]{kingma2014}
\begin{binproceedings}[author]
\bauthor{\bsnm{Kingma},~\bfnm{D.~P.}\binits{D.~P.}} \AND
  \bauthor{\bsnm{Welling},~\bfnm{M.}\binits{M.}}
(\byear{2014}).
\btitle{Auto-encoding variational {Bayes}}.
In \bbooktitle{International Conference on Learning Representations}.
\end{binproceedings}
\endbibitem

\bibitem[\protect\citeauthoryear{Klys, Bettencourt and
  Duvenaud}{2018}]{klys2018joint}
\begin{barticle}[author]
\bauthor{\bsnm{Klys},~\bfnm{J.}\binits{J.}},
  \bauthor{\bsnm{Bettencourt},~\bfnm{J.}\binits{J.}} \AND
  \bauthor{\bsnm{Duvenaud},~\bfnm{D.}\binits{D.}}
(\byear{2018}).
\btitle{Joint Importance Sampling for Variational Inference}.
\bjournal{Openreview preprint}.
\end{barticle}
\endbibitem

\bibitem[\protect\citeauthoryear{Le et~al.}{2018}]{anh2018}
\begin{binproceedings}[author]
\bauthor{\bsnm{Le},~\bfnm{T.~A.}\binits{T.~A.}},
  \bauthor{\bsnm{Igl},~\bfnm{M.}\binits{M.}},
  \bauthor{\bsnm{Rainforth},~\bfnm{T.}\binits{T.}},
  \bauthor{\bsnm{Jin},~\bfnm{T.}\binits{T.}} \AND
  \bauthor{\bsnm{Wood},~\bfnm{F.}\binits{F.}}
(\byear{2018}).
\btitle{Auto-Encoding Sequential {Monte Carlo}}.
In \bbooktitle{International Conference on Learning Representations}.
\end{binproceedings}
\endbibitem

\bibitem[\protect\citeauthoryear{Le et~al.}{2020}]{le2020revisiting}
\begin{binproceedings}[author]
\bauthor{\bsnm{Le},~\bfnm{T.~A.}\binits{T.~A.}},
  \bauthor{\bsnm{Kosiorek},~\bfnm{A.~R.}\binits{A.~R.}},
  \bauthor{\bsnm{Siddharth},~\bfnm{N.}\binits{N.}},
  \bauthor{\bsnm{Teh},~\bfnm{Y.~W.}\binits{Y.~W.}} \AND
  \bauthor{\bsnm{Wood},~\bfnm{F.}\binits{F.}}
(\byear{2020}).
\btitle{Revisiting reweighted wake-sleep for models with stochastic control
  flow}.
In \bbooktitle{Uncertainty in Artificial Intelligence}
\bpages{1039--1049}.
\end{binproceedings}
\endbibitem

\bibitem[\protect\citeauthoryear{Leskel{\"a} and
  Vihola}{2017}]{leskela2017conditional}
\begin{barticle}[author]
\bauthor{\bsnm{Leskel{\"a}},~\bfnm{L.}\binits{L.}} \AND
  \bauthor{\bsnm{Vihola},~\bfnm{M.}\binits{M.}}
(\byear{2017}).
\btitle{Conditional convex orders and measurable martingale couplings}.
\bjournal{Bernoulli}
\bvolume{23}
\bpages{2784--2807}.
\end{barticle}
\endbibitem

\bibitem[\protect\citeauthoryear{Li{\'e}vin et~al.}{2020}]{lievin2020optimal}
\begin{barticle}[author]
\bauthor{\bsnm{Li{\'e}vin},~\bfnm{V.}\binits{V.}},
  \bauthor{\bsnm{Dittadi},~\bfnm{A.}\binits{A.}},
  \bauthor{\bsnm{Christensen},~\bfnm{A.}\binits{A.}} \AND
  \bauthor{\bsnm{Winther},~\bfnm{O.}\binits{O.}}
(\byear{2020}).
\btitle{Optimal Variance Control of the Score-Function Gradient Estimator for
  Importance-Weighted Bounds}.
\bjournal{Advances in Neural Information Processing Systems}
\bvolume{33}
\bpages{16591--16602}.
\end{barticle}
\endbibitem

\bibitem[\protect\citeauthoryear{Maddison et~al.}{2017}]{maddison2017}
\begin{binproceedings}[author]
\bauthor{\bsnm{Maddison},~\bfnm{C.~J.}\binits{C.~J.}},
  \bauthor{\bsnm{Lawson},~\bfnm{J.}\binits{J.}},
  \bauthor{\bsnm{Tucker},~\bfnm{G.}\binits{G.}},
  \bauthor{\bsnm{Heess},~\bfnm{N.}\binits{N.}},
  \bauthor{\bsnm{Norouzi},~\bfnm{M.}\binits{M.}},
  \bauthor{\bsnm{Mnih},~\bfnm{A.}\binits{A.}},
  \bauthor{\bsnm{Doucet},~\bfnm{A.}\binits{A.}} \AND
  \bauthor{\bsnm{Teh},~\bfnm{Y.}\binits{Y.}}
(\byear{2017}).
\btitle{Filtering variational objectives}.
In \bbooktitle{Advances in Neural Information Processing Systems}
\bpages{6573--6583}.
\end{binproceedings}
\endbibitem

\bibitem[\protect\citeauthoryear{Marshall, Olkin and
  Arnold}{2011}]{marshall2011inequalities}
\begin{bbook}[author]
\bauthor{\bsnm{Marshall},~\bfnm{A.~W.}\binits{A.~W.}},
  \bauthor{\bsnm{Olkin},~\bfnm{I.}\binits{I.}} \AND
  \bauthor{\bsnm{Arnold},~\bfnm{B.}\binits{B.}}
(\byear{2011}).
\btitle{Inequalities: Theory of Majorization and Its Applications}.
\bpublisher{Springer}.
\end{bbook}
\endbibitem

\bibitem[\protect\citeauthoryear{Marshall and
  Proschan}{1965}]{marshall1965inequality}
\begin{barticle}[author]
\bauthor{\bsnm{Marshall},~\bfnm{A.~W.}\binits{A.~W.}} \AND
  \bauthor{\bsnm{Proschan},~\bfnm{F.}\binits{F.}}
(\byear{1965}).
\btitle{An inequality for convex functions involving majorization}.
\bjournal{Journal of Mathematical Analysis and Applications}
\bvolume{12}
\bpages{87--90}.
\end{barticle}
\endbibitem

\bibitem[\protect\citeauthoryear{Mattei and Frellsen}{2019}]{mattei2019}
\begin{binproceedings}[author]
\bauthor{\bsnm{Mattei},~\bfnm{P.~A.}\binits{P.~A.}} \AND
  \bauthor{\bsnm{Frellsen},~\bfnm{J.}\binits{J.}}
(\byear{2019}).
\btitle{{MIWAE}: Deep Generative Modelling and Imputation of Incomplete Data
  Sets}.
In \bbooktitle{International Conference on Machine Learning}
\bpages{4413--4423}.
\end{binproceedings}
\endbibitem

\bibitem[\protect\citeauthoryear{Meyer and
  Strulovici}{2012}]{meyer2012increasing}
\begin{barticle}[author]
\bauthor{\bsnm{Meyer},~\bfnm{M.}\binits{M.}} \AND
  \bauthor{\bsnm{Strulovici},~\bfnm{B.}\binits{B.}}
(\byear{2012}).
\btitle{Increasing interdependence of multivariate distributions}.
\bjournal{Journal of Economic Theory}
\bvolume{147}
\bpages{1460--1489}.
\end{barticle}
\endbibitem

\bibitem[\protect\citeauthoryear{Meyer and Strulovici}{2015}]{meyer2013}
\begin{btechreport}[author]
\bauthor{\bsnm{Meyer},~\bfnm{M.}\binits{M.}} \AND
  \bauthor{\bsnm{Strulovici},~\bfnm{B.}\binits{B.}}
(\byear{2015}).
\btitle{Beyond Correlation: Measuring Interdependence Through
  Complementarities}
\btype{Economics Series Working Papers} No. \bnumber{655},
\bpublisher{University of Oxford, Department of Economics}.
\end{btechreport}
\endbibitem

\bibitem[\protect\citeauthoryear{Mnih and Rezende}{2016}]{mnih2016}
\begin{binproceedings}[author]
\bauthor{\bsnm{Mnih},~\bfnm{A.}\binits{A.}} \AND
  \bauthor{\bsnm{Rezende},~\bfnm{D.}\binits{D.}}
(\byear{2016}).
\btitle{Variational Inference for {Monte Carlo} Objectives}.
In \bbooktitle{International Conference on Machine Learning}
\bpages{2188--2196}.
\end{binproceedings}
\endbibitem

\bibitem[\protect\citeauthoryear{M{\"u}ller}{1997}]{muller1997stop}
\begin{barticle}[author]
\bauthor{\bsnm{M{\"u}ller},~\bfnm{A.}\binits{A.}}
(\byear{1997}).
\btitle{Stop-loss order for portfolios of dependent risks}.
\bjournal{Insurance: Mathematics and Economics}
\bvolume{21}
\bpages{219--223}.
\end{barticle}
\endbibitem

\bibitem[\protect\citeauthoryear{M{\"u}ller and Scarsini}{2000}]{muller2000}
\begin{barticle}[author]
\bauthor{\bsnm{M{\"u}ller},~\bfnm{A.}\binits{A.}} \AND
  \bauthor{\bsnm{Scarsini},~\bfnm{M.}\binits{M.}}
(\byear{2000}).
\btitle{Some remarks on the supermodular order}.
\bjournal{Journal of Multivariate Analysis}
\bvolume{73}
\bpages{107--119}.
\end{barticle}
\endbibitem

\bibitem[\protect\citeauthoryear{Naesseth
  et~al.}{2018}]{naesseth2018variational}
\begin{binproceedings}[author]
\bauthor{\bsnm{Naesseth},~\bfnm{C.}\binits{C.}},
  \bauthor{\bsnm{Linderman},~\bfnm{S.}\binits{S.}},
  \bauthor{\bsnm{Ranganath},~\bfnm{R.}\binits{R.}} \AND
  \bauthor{\bsnm{Blei},~\bfnm{D.}\binits{D.}}
(\byear{2018}).
\btitle{Variational Sequential {Monte Carlo}}.
In \bbooktitle{International Conference on Artificial Intelligence and
  Statistics}
\bpages{968--977}.
\end{binproceedings}
\endbibitem

\bibitem[\protect\citeauthoryear{Nowozin}{2018}]{nowozin2018debiasing}
\begin{binproceedings}[author]
\bauthor{\bsnm{Nowozin},~\bfnm{S.}\binits{S.}}
(\byear{2018}).
\btitle{Debiasing Evidence Approximations: On Importance-weighted Autoencoders
  and Jackknife Variational Inference}.
In \bbooktitle{International Conference on Learning Representations}.
\end{binproceedings}
\endbibitem

\bibitem[\protect\citeauthoryear{Owen}{2013}]{mcbook}
\begin{bbook}[author]
\bauthor{\bsnm{Owen},~\bfnm{A.~B.}\binits{A.~B.}}
(\byear{2013}).
\btitle{{Monte Carlo} theory, methods and examples}.
\end{bbook}
\endbibitem

\bibitem[\protect\citeauthoryear{Rainforth et~al.}{2018a}]{rainforth18}
\begin{binproceedings}[author]
\bauthor{\bsnm{Rainforth},~\bfnm{T.}\binits{T.}},
  \bauthor{\bsnm{Kosiorek},~\bfnm{A.}\binits{A.}},
  \bauthor{\bsnm{Le},~\bfnm{Tuan~A.}\binits{T.~A.}},
  \bauthor{\bsnm{Maddison},~\bfnm{C.}\binits{C.}},
  \bauthor{\bsnm{Igl},~\bfnm{M.}\binits{M.}},
  \bauthor{\bsnm{Wood},~\bfnm{F.}\binits{F.}} \AND
  \bauthor{\bsnm{Teh},~\bfnm{Y.~W.}\binits{Y.~W.}}
(\byear{2018}a).
\btitle{Tighter Variational Bounds are Not Necessarily Better}.
In \bbooktitle{Proceedings of the 35th International Conference on Machine
  Learning}.
\bseries{Proceedings of Machine Learning Research}
\bpages{4277--4285}.
\end{binproceedings}
\endbibitem

\bibitem[\protect\citeauthoryear{Rainforth
  et~al.}{2018b}]{rainforth2018nesting}
\begin{binproceedings}[author]
\bauthor{\bsnm{Rainforth},~\bfnm{T.}\binits{T.}},
  \bauthor{\bsnm{Cornish},~\bfnm{R.}\binits{R.}},
  \bauthor{\bsnm{Yang},~\bfnm{H.}\binits{H.}},
  \bauthor{\bsnm{Warrington},~\bfnm{A.}\binits{A.}} \AND
  \bauthor{\bsnm{Wood},~\bfnm{F.}\binits{F.}}
(\byear{2018}b).
\btitle{On Nesting {Monte Carlo} Estimators}.
In \bbooktitle{International Conference on Machine Learning}
\bpages{4267--4276}.
\end{binproceedings}
\endbibitem

\bibitem[\protect\citeauthoryear{Ren, Zhao and Ermon}{2019}]{ren2019}
\begin{binproceedings}[author]
\bauthor{\bsnm{Ren},~\bfnm{H.}\binits{H.}},
  \bauthor{\bsnm{Zhao},~\bfnm{S.}\binits{S.}} \AND
  \bauthor{\bsnm{Ermon},~\bfnm{S.}\binits{S.}}
(\byear{2019}).
\btitle{Adaptive Antithetic Sampling for Variance Reduction}.
In \bbooktitle{International Conference on Machine Learning}
\bpages{5420--5428}.
\end{binproceedings}
\endbibitem

\bibitem[\protect\citeauthoryear{Rezende, Mohamed and
  Wierstra}{2014}]{rezende2014}
\begin{binproceedings}[author]
\bauthor{\bsnm{Rezende},~\bfnm{D.}\binits{D.}},
  \bauthor{\bsnm{Mohamed},~\bfnm{S.}\binits{S.}} \AND
  \bauthor{\bsnm{Wierstra},~\bfnm{D.}\binits{D.}}
(\byear{2014}).
\btitle{Stochastic Backpropagation and Approximate Inference in Deep Generative
  Models}.
In \bbooktitle{Proceedings of the 31st International Conference on Machine
  Learning}
\bpages{1278--1286}.
\end{binproceedings}
\endbibitem

\bibitem[\protect\citeauthoryear{Robinson}{2015}]{robinson2015}
\begin{barticle}[author]
\bauthor{\bsnm{Robinson},~\bfnm{G.~K.}\binits{G.~K.}}
(\byear{2015}).
\btitle{Practical computing for finite moment log-stable distributions to model
  financial risk}.
\bjournal{Statistics and Computing}
\bvolume{25}
\bpages{1233--1246}.
\end{barticle}
\endbibitem

\bibitem[\protect\citeauthoryear{R{\"u}schendorf}{2013}]{ruschendorf2013mathematical}
\begin{bbook}[author]
\bauthor{\bsnm{R{\"u}schendorf},~\bfnm{L.}\binits{L.}}
(\byear{2013}).
\btitle{Mathematical risk analysis}.
\bpublisher{Springer Series in Operations Research and Financial Engineering}.
\end{bbook}
\endbibitem

\bibitem[\protect\citeauthoryear{Salimbeni et~al.}{2019}]{salimbeni2019deep}
\begin{barticle}[author]
\bauthor{\bsnm{Salimbeni},~\bfnm{H.}\binits{H.}},
  \bauthor{\bsnm{Dutordoir},~\bfnm{V.}\binits{V.}},
  \bauthor{\bsnm{Hensman},~\bfnm{J.}\binits{J.}} \AND
  \bauthor{\bsnm{Deisenroth},~\bfnm{M.~P.}\binits{M.~P.}}
(\byear{2019}).
\btitle{Deep gaussian processes with importance-weighted variational
  inference}.
\bjournal{arXiv preprint arXiv:1905.05435}.
\end{barticle}
\endbibitem

\bibitem[\protect\citeauthoryear{Shaked and Shanthikumar}{2007}]{shaked2007}
\begin{bbook}[author]
\bauthor{\bsnm{Shaked},~\bfnm{M.}\binits{M.}} \AND
  \bauthor{\bsnm{Shanthikumar},~\bfnm{J.~G.}\binits{J.~G.}}
(\byear{2007}).
\btitle{Stochastic orders}.
\bpublisher{Springer Science \& Business Media}.
\end{bbook}
\endbibitem

\bibitem[\protect\citeauthoryear{Speiser et~al.}{2017}]{speiser2017}
\begin{binproceedings}[author]
\bauthor{\bsnm{Speiser},~\bfnm{A.}\binits{A.}},
  \bauthor{\bsnm{Yan},~\bfnm{J.}\binits{J.}},
  \bauthor{\bsnm{Archer},~\bfnm{E.~W.}\binits{E.~W.}},
  \bauthor{\bsnm{Buesing},~\bfnm{L.}\binits{L.}},
  \bauthor{\bsnm{Turaga},~\bfnm{S.~C.}\binits{S.~C.}} \AND
  \bauthor{\bsnm{Macke},~\bfnm{J.~H.}\binits{J.~H.}}
(\byear{2017}).
\btitle{Fast amortized inference of neural activity from calcium imaging data
  with variational autoencoders}.
In \bbooktitle{Advances in Neural Information Processing Systems}
\bpages{4024--4034}.
\end{binproceedings}
\endbibitem

\bibitem[\protect\citeauthoryear{Stummer and
  Vajda}{2010}]{stummer2010divergences}
\begin{barticle}[author]
\bauthor{\bsnm{Stummer},~\bfnm{W.}\binits{W.}} \AND
  \bauthor{\bsnm{Vajda},~\bfnm{I.}\binits{I.}}
(\byear{2010}).
\btitle{On divergences of finite measures and their applicability in statistics
  and information theory}.
\bjournal{Statistics}
\bvolume{44}
\bpages{169--187}.
\end{barticle}
\endbibitem

\bibitem[\protect\citeauthoryear{Tan and Friel}{2020}]{tan2020bayesian}
\begin{barticle}[author]
\bauthor{\bsnm{Tan},~\bfnm{L.~S.~L.}\binits{L.~S.~L.}} \AND
  \bauthor{\bsnm{Friel},~\bfnm{N.}\binits{N.}}
(\byear{2020}).
\btitle{Bayesian variational inference for exponential random graph models}.
\bjournal{Journal of Computational and Graphical Statistics}
\bvolume{in press}.
\end{barticle}
\endbibitem

\bibitem[\protect\citeauthoryear{Tucker et~al.}{2019}]{tucker2018}
\begin{binproceedings}[author]
\bauthor{\bsnm{Tucker},~\bfnm{G.}\binits{G.}},
  \bauthor{\bsnm{Lawson},~\bfnm{D.}\binits{D.}},
  \bauthor{\bsnm{Gu},~\bfnm{S.}\binits{S.}} \AND
  \bauthor{\bsnm{Maddison},~\bfnm{C.~J.}\binits{C.~J.}}
(\byear{2019}).
\btitle{Doubly Reparameterized Gradient Estimators for Monte Carlo Objectives}.
In \bbooktitle{International Conference on Learning Representations}.
\end{binproceedings}
\endbibitem

\bibitem[\protect\citeauthoryear{Vehtari et~al.}{2019}]{vehtari2019pareto}
\begin{barticle}[author]
\bauthor{\bsnm{Vehtari},~\bfnm{A.}\binits{A.}},
  \bauthor{\bsnm{Simpson},~\bfnm{D.}\binits{D.}},
  \bauthor{\bsnm{Gelman},~\bfnm{A.}\binits{A.}},
  \bauthor{\bsnm{Yao},~\bfnm{Y.}\binits{Y.}} \AND
  \bauthor{\bsnm{Gabry},~\bfnm{J.}\binits{J.}}
(\byear{2019}).
\btitle{Pareto smoothed importance sampling}.
\bjournal{arXiv preprint arXiv:1507.02646}.
\end{barticle}
\endbibitem

\bibitem[\protect\citeauthoryear{Wu, Goodman and Ermon}{2019}]{wu2019}
\begin{binproceedings}[author]
\bauthor{\bsnm{Wu},~\bfnm{M.}\binits{M.}},
  \bauthor{\bsnm{Goodman},~\bfnm{N.}\binits{N.}} \AND
  \bauthor{\bsnm{Ermon},~\bfnm{S.}\binits{S.}}
(\byear{2019}).
\btitle{Differentiable Antithetic Sampling for Variance Reduction in Stochastic
  Variational Inference}.
In \bbooktitle{Proceedings of the 22nd International Conference on Artificial
  Intelligence and Statistics}
\bpages{2877--2886}.
\end{binproceedings}
\endbibitem

\end{thebibliography}

\end{document}